%% file: main.tex
\documentclass{article}

    \usepackage[preprint]{neurips_2022}

\usepackage{microtype}
\usepackage{graphicx}
\usepackage{booktabs} 
\usepackage{hyperref}

\usepackage{amsmath}
\usepackage{amssymb}
\usepackage{mathtools}
\usepackage{amsthm}
\usepackage{svg}

\usepackage[capitalize,noabbrev]{cleveref}

\theoremstyle{plain}

\theoremstyle{definition}

\theoremstyle{remark}

\usepackage[textsize=tiny]{todonotes}

\usepackage[utf8]{inputenc} 
\usepackage[T1]{fontenc}    
\usepackage{hyperref}       
\usepackage{url}            
\usepackage{amsfonts}       
\usepackage{nicefrac}       
\usepackage{xcolor}         
\usepackage{mathabx}
\usepackage{subfig,wrapfig}
\usepackage{tikz}
\usepackage{bm}
\usepackage{graphicx}
\usepackage{colortbl}
\usepackage{txfonts}
\usepackage{wrapfig}
\usepackage[normalem]{ulem} 

\DeclareUnicodeCharacter{2028}{\linebreak}

\newif\ifdraft
\drafttrue

\newtheorem{thm}{Theorem}
\newtheorem{claim}{Claim}
\newtheorem{observation}{Observation}

\newcommand*{\centernot}{%
  \mathpalette\@centernot
}
\def\@centernot#1#2{%
  \mathrel{%
    \rlap{%
      \settowidth\dimen@{$\m@th#1{#2}$}%
      \kern.5\dimen@
      \settowidth\dimen@{$\m@th#1=$}%
      \kern-.5\dimen@
      $\m@th#1\not$%
    }%
    {#2}%
  }%
}

\newcommand\independent{\protect\mathpalette{\protect\independenT}{\perp}}
\def\independenT#1#2{\mathrel{\rlap{$#1#2$}\mkern2mu{#1#2}}}

\newcommand{\indep}{\Perp}

\newcommand{\eps}{\varepsilon}
\renewcommand{\epsilon}{\varepsilon}

\newcommand{\gapsuff}[1]{\Delta^{\mathrm{suff}}_{#1}}
\newcommand{\gapsuffmax}{\Delta^{\mathrm{suff}}_{\mathrm{max}}}
\newcommand{\gapsep}[1]{\Delta^{\mathrm{sep}}_{#1}}
\newcommand{\gapsepmax}{\Delta^{\mathrm{sep}}_{\mathrm{max}}}
\newcommand{\allgapmax}{\mathrm{max} (\gapsuff{}, \gapsep{})}
\newcommand{\allgapmaxbinarydiff}{\mathrm{max} (\gapsuff{binary}, \gapsep{binary})}
\renewcommand{\Pr}{\mathbb{P}}
\newcommand{\PPV}{\textsf{PPV}}
\newcommand{\NPV}{\textsf{NPV}}
\newcommand{\TPR}{\textsf{TPR}}
\newcommand{\FPR}{\textsf{FPR}}
\newcommand{\FOR}{\textsf{FOR}}
\newcommand{\FDR}{\textsf{FDR}}
\newcommand{\TNR}{\textsf{TNR}}
\newcommand{\FNR}{\textsf{FNR}}

\title{Beyond Impossibility: Balancing Sufficiency, Separation and Accuracy}

\author{%
 Limor Gultchin \\
  Department of Computer Science\\
  University of Oxford\\
  \texttt{limor.gultchin@gmail.com} \\
   \And
    Vincent Cohen-Addad \\
   Google \\
   \texttt{cohenaddad@google.com} \\
   \And
   Sophie Giffard-Roisin \\
   University of Grenoble-Alpes \\
   \texttt{sophie.giffard@univ-grenoble-alpes.fr} \\
   \And
   Varun Kanade \\
   Department of Computer Science\\
   University of Oxford \\
   \texttt{varunk@cs.ox.ac.uk} \\
   \And
   Frederik Mallmann-Trenn \\
   Department of Computer Science\\
   King's College London \\
   \texttt{frederik.mallmann-trenn@kcl.ac.uk} \\
}

\begin{document}
\maketitle

\begin{abstract}
\input{sections/0.Abstract}
\end{abstract}

\section{Introduction}

\input{sections/1.Introduction}

\section{Background}\label{sec:background}

\input{sections/2.Background}

\section{Theoretical Contribution}\label{sec:theory_results}
\input{sections/3.Theory}

\section{Methods}\label{sec:experiments_methods}

\input{sections/4.Method}

\section{Experiments}\label{sec:experiments_results}

\input{sections/7.Results}
\section{Conclusion}\label{sec:conclusion}
\input{sections/6.Conclusion}




\bibliographystyle{plainnat}
\bibliography{main}

\newpage
\appendix
\onecolumn
\section{Appendix}

\input{sections/A.Appendix}

\end{document}

%% file: sections/0.Abstract.tex
Among the various aspects of algorithmic fairness studied in recent years, the tension between satisfying both \textit{sufficiency} and \textit{separation} -- e.g. the ratios of positive or negative predictive values, and false positive or false negative rates across groups -- has received much attention. Following a debate sparked by COMPAS, a criminal justice predictive system, the academic community has responded by laying out important theoretical understanding, showing that one cannot achieve both with an imperfect predictor when there is no equal distribution of labels across the groups. In this paper, we shed more light on what might be still possible beyond the impossibility -- the existence of a  trade-off means we should aim to find a good balance within it. After refining the existing theoretical result, we propose an objective that aims to balance \textit{sufficiency} and \textit{separation} measures, while maintaining similar accuracy levels. We show the use of such an objective in two empirical case studies, one involving a multi-objective framework, and the other fine-tuning of a model pre-trained for accuracy. We show promising results, where better trade-offs are achieved compared to existing alternatives.

%% file: sections/1.Introduction.tex
The advance of automatic decision making into social domains has made training and auditing Machine Learning (ML) systems that treat sensitive groups fairly a crucial research area. Fair and Responsible ML has been a growing research area in recent years \cite{dwork2012fairness, barocas-hardt-narayanan}, and academic research has fostered important debates around the use of ML for predicting sensitive information, such as for example recidivism in criminal justice systems.
In 2016, ProPublica made an important contribution to these debates by analyzing a criminal justice system's score allocation, called COMPAS,  and showing that it had twice as high False-Positive Rate ($\FPR$) (44.85 vs. 27.99 on the non-violent COMPAS scores) for black defendants compared to white defendants, while False-Negative Rate ($\FNR$) was twice as high for white defendants compared to black defendants (47.2 vs. 23.45 on the non-violent COMPAS scores) \citep{larson2016we}.
Northpointe, the developer of COMPAS, launched a response showing
that a different fairness measure, called calibration, was in fact satisfied by the system~\citep{flores2016false}. This immediately opened a basic question: what is the \emph{right}
fairness criterion for ML systems?


\begin{figure}
    \centering
    \includegraphics[width=0.3\textwidth]{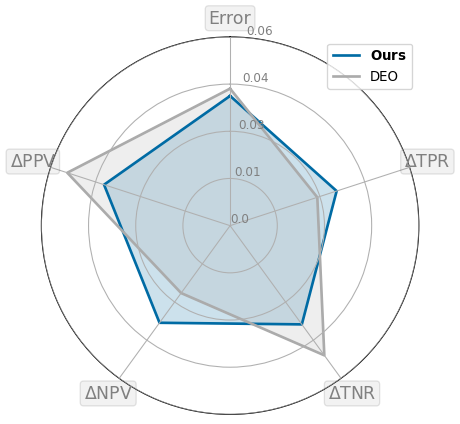}
    \caption{We propose, conceptually, that fairness criteria should aim to balance \textit{separation} and \textit{sufficiency} measures. We propose a criterion
    and show it often leads to better tradeoffs. In the plot above, we show how the method we call \textbf{Ours} 
    fares against a Difference of Equality of Opportunity (DEO) constraint \cite{Hardt_NIPS2016_9d268236}, which only involves one \textit{separation} measure, used in a multi-objective method on the dataset \texttt{Color-MNIST} in terms of error (top axis), \textit{separation} violations (TPR, FPR absolute group differences) (right), and \textit{sufficiency} violations (PPV, NPV absolute group differences) (left).}
    \vspace{-0.6cm}
    \label{fig:radar}
\end{figure}

The academic community jumped into the discussion and contributed a much needed theoretical understanding on these initial results. Multiple researchers showed that the two fairness criteria used in the COMPAS debate, calibration and error rates, were not coincidentally emerging in opposition to each other – when the labels are not equally spread across sensitive groups (as is often the case in criminal justice and other applications), one cannot have both perfect calibration and equal error rates at the same time (unless the predictor is perfect), as proven in a seminal impossibility theorem \cite{chouldechova2017fair, kleinberg-et-al:LIPIcs:2017:8156, corbett2017algorithmic}. The data distribution together with an ML classifier generates a probability distribution on the triple $(\hat{Y}, Y, A)$, where $\hat{Y}$ is the predicted label, $Y$ the true label, and $A$ the protected attribute.  The aforementioned fairness measures correspond to two types of conditional independence among these three quantities: \textit{Separation}, $\hat{Y} \independent A~|~ Y$ (related to balance for the positive/negative class, equalized odds, equality of opportunity, conditional procedure accuracy equality etc.), and \textit{Sufficiency} $Y \independent A~|~ \hat{Y}$ (related to calibration within groups, test-fair score, conditional use accuracy equality, etc.).\footnote{See \citep{barocas-hardt-narayanan} for a comprehensive survey, and the `dictionary of criteria' in Chapter 2 for a handy summary.} 

The impossibility result ruling out equal error rates and perfect calibration simultaneously in practice has had an important and beneficial impact on the development of the fairness literature. In a certain sense, this result puts the COMPAS debate to rest: Northpointe and ProPublica are interested in different  incompatible fairness criteria and the choice of Northpointe to focus on calibration can only be questioned outside the realm of mathematical reasoning.
Another consequence of the theorem has been that researchers have advocated picking a single arbitrary fairness definition and design machine learning models that satisfy this fairness definition.\footnote{For example, this one measure can be chosen to be the absolute difference across sensitive groups in one of the possible eight values implied by a standard confusion matrix for a binary classification problem.} 
Multiple papers thus either 1) recommend practitioners to stick to one group of fairness measures and disregard the others \cite{Pleiss_NIPS2017_b8b9c74a}; or 2) suggest that certain results are worse when trying to uphold more than one type of fairness constraint \cite{pmlr-v161-padh21a}, or 3) promote alternative approaches to fairness, partially to bypass this issue (e.g. Counterfactual Fairness \cite{kusner2017counterfactual}). 

Nevertheless, these resolutions are somewhat unsatisfactory: counterfactual fairness, for one, is a very principled way to look at the problem but requires knowledge of the data generation process, and thus may require strong assumptions that make practical implementation more difficult; the other approaches, focusing on one measure of fairness without requiring any guarantees on the others may, for example, allow claims regarding a model's fairness with respect to one (possibly ``cherry-picked'') criterion while being particularly unfair with respect to several other. Phrased more positively, we ask: Given a desired quality level for one fairness criterion (hence seen as a constraint), how well can we optimize the other criteria? 
We demonstrate that a better balance can be achieved by simultaneously optimizing multiple criteria (cf. Figure~\ref{fig:radar}). Concretely, this problem can also be illustrated with the COMPAS debate: Acknowledging that the COMPAS system achieves a certain (approximate) calibration as claimed by Northpointe, one may ask how well the system performs in terms of the equal-error-rates measure compared to other systems achieving a similar level of calibration. In other words, we could be satisfied with the COMPAS system if, given the level of calibration achieved, it is the ``best'' also with respect to equal-error-rates among all models that have a similar accuracy and calibration.

\begin{figure*}[!t]
    \centering
    \includegraphics[width=\textwidth]{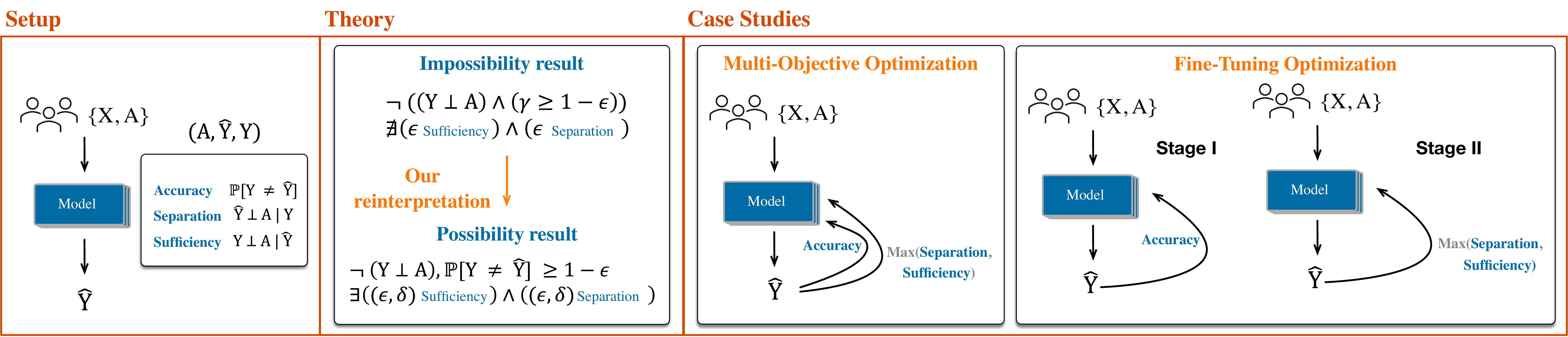}
    \caption{A general description of our contributions.}
    \vspace{-0.5cm}
    \label{fig:case_study_schema}
\end{figure*}

\subsection{Our contributions}
In this work, we introduce an approximate view on the calibration and equal-error-rates criteria that is particularly relevant in practical settings. This leads us to analyze how these quantities can be traded-off and how well can one quantity be approximated under the constraint that the other remains under some desired threshold. We show how to
achieve a satisfying balance between approximate versions of these fairness criteria and accuracy as a whole.
We focus on the measures $\gapsuff{y, \hat{y}, a} = |\Pr[Y=y \mid  \hat{Y}=\hat{y}, A=a] - \Pr[Y=y \mid \hat{Y}=y]|$ and $\gapsep{\hat{y}, y, a} = |\Pr[\hat{Y}=y \mid Y=y, A=a] - \Pr[\hat{Y}=y \mid Y=y]|$. We treat $\gapsuffmax = \max_{y, \hat{y}, a} \gapsuff{y, \hat{y}, a}$ and $\gapsepmax = \max_{\hat{y}, y, a} \gapsep{\hat{y}, y, a}$ as representing deviation from perfect \textit{sufficiency} and \textit{separation} and minimize these quantities. 

In Section~\ref{sec:background} we provide further details on related work. Section~\ref{sec:theory_results} delves into the impossibility
result, and offers a refinement of the theorem in terms of the above criteria:
\textit{separation}, \textit{sufficiency} and accuracy of the predictor. Our version of the
theorem can be interpreted as an approximate possibility result: if a
classifier can hold both \textit{sufficiency} and \textit{separation} approximately (i.e.: under
some prescribed threshold), it must be fairly accurate.\footnote{Assuming base
rates are not almost equally distributed across groups (a very unlikely
possibility under most fairness applications).} Stated alternatively, if one has
a fairly accurate classifier, it is possible to hold both \textit{sufficiency}
and \textit{separation} violations at small values. We therefore argue that we
should indeed consider what can be achieved in terms of the above measures 
beyond the strict impossibility. The implications of our theorem for the COMPAS
debate is the following: Given the accuracy and the level of
\textit{separation} that are achieved by Northpointe's predictor, our theory prescribes that a certain a minimal level of \textit{sufficiency} can be achieved. One may thus ask whether
Northpointe's predictor achieves this minimal level of \textit{sufficiency}.

Equipped with this theoretical insight, we design new loss functions for training classifiers that have both good accuracy 
and minimal fairness requirements for \textit{sufficiency} and \textit{separation} criteria (Sections~\ref{sec:experiments_methods} and \ref{sec:experiments_results}).
We propose the use of our new loss functions via an in-processing multi-objective framework \cite{pmlr-v161-padh21a}, as well as using one of them as a post-processing technique, fine-tuning an existing model with a fairness-dedicated objective. 
We show that our new loss function $\allgapmax$ yields classifiers achieving better trade-offs in terms of accuracy, \textit{separation} and \textit{sufficiency} compared to existing alternatives across the 4 datasets we studied: \texttt{COMPAS} \cite{larson2016we}, \texttt{Adult Income} \cite{Kohavi_10.5555/3001460.3001502}, \texttt{NELS} \cite{ingels1990national} and \texttt{Color-MNIST} \cite{Lecun1998_726791}.\footnote{\textit{Color-MNIST} is an adaptation of the original MNIST dataset to this setting, see full description in \ref{sec:experiments_datasets}.}

%% file: sections/2.Background.tex
\textbf{Constrained optimization under fairness notions.} Multiple works considered forms of constrained optimization of predictors to achieve one of the fairness notions we consider in this work \citep{zafar2017fairness2, wu2019convexity, donini2018empirical}. However, they consider achieving either the $\gapsuff{}$ or the $\gapsep{}$ family.
Previous authors moreover studied the \textbf{the fairness and accuracy trade-off via Pareto Frontiers} \citep{kearns2018preventing, kearns2019ethical}. Instead, we focus on predictors holding both fairness notions of $\gapsuff{}$ and $\gapsep{}$ in an approximate fashion, and propose optimization methods inspired by this idea. 
There were also more recent works on \textbf{Fairness relaxations}
\citep{lohaus2020too}, \textbf{Group and subgroup Fairness}
\citep{martinez2020minimax, diana2020convergent, kearns2018preventing, Yang_DBLP:journals/corr/abs-2006-13485} or \textbf{fair transformations of scores} \cite{Jiang_pmlr-v115-jiang20a}. While they all raise important and interesting points about how to hold one chosen family of notions in a properly relaxed version which comes with guarantees, or exploring subgroup fairness, 
our focus is in a way more fundamental, going back to the setting of the original impossibility result. 
One other set of methods that could enable the balancing of multiple fairness objectives are \textbf{Multi-objective approaches to fairness} \cite{Celis_10.1145/3287560.3287586, pmlr-v161-padh21a, Ruchte_DBLP:journals/corr/abs-2103-13392}. While the above works could handle the challenge of balancing \textit{separation} and \textit{sufficiency}, they do not directly address this question, and instead focus their experiments on the combination of other objectives (e.g., Demographic Parity (DP) and Equality of Opportunity (EOP) \cite{Hardt_NIPS2016_9d268236}). We believe this might be due to the direct or indirect influence of the impossibility result: \cite{kleinberg-et-al:LIPIcs:2017:8156} and \cite{chouldechova2017fair} are each cited over 1,000 times so far. Furthermore, the legacy it has left on the literature goes beyond the original result itself, including \cite{Pleiss_NIPS2017_b8b9c74a} which directly recommends picking one of the violation groups to focus on and give up on the other.\footnote{We provide more examples of the use of the impossibility result in the literature in the appendix.} We do not know of papers that try to balance both \textit{sufficiency} and \textit{separation} directly. 
There are \textbf{other works that attempt to look beyond the impossibility results} via other means. \cite{lazarreich_et_al:LIPIcs.FORC.2021.4} show how to provide scores (model output pre-threshold) satisfying calibration, and label predictions (post-threshold) satisfying equal error rates. \cite{Blum_DBLP:journals/corr/abs-1912-01094} suggest that when starting with a biased-dataset (according to definitions they provide), 
fairness-motivated constraints can lead to better results even if one cares only about accuracy. Finally, other works looked at \textbf{the fairness and accuracy trade-off} \cite{Menon_pmlr-v81-menon18a, chen2018classifier, Dutta_pmlr-v119-dutta20a, Wick_unlock_NEURIPS2019_373e4c5d}. 


%% file: sections/3.Theory.tex
\subsection{Separation and Sufficiency}\label{sec:sep_suff}

In the following, let $A$, $Y$ and $\hat{Y}$ denote the random variable representing the
sensitive attribute, the ground truth label and the outcome of a prediction
model. We will denote by $a$, $y$ and $\hat{y}$ the values taken by these
random variables; for simplicity we will assume that each of these can only
take the values $\{0, 1\}$. 


\emph{Separation} refers to the conditional independence of $\hat{Y}$ and $A$
given $Y$, i.e. $\hat{Y} \Perp A | Y$. This in particular, implies that for
any $a, y, \hat{y}$, $\Pr[\hat{Y}= \hat{y} | Y=y, A=a ] = \Pr[\hat{Y} = \hat{y} | Y = y]$. For an assignment of values $y$, $\hat{y}$ and $a$, we denote the \textit{separation} gap by $\gapsep{\hat{y}, y, a} = |\Pr[\hat{Y} = \hat{y} | Y = y, A=a] - \Pr[\hat{Y}=y|Y=y]|$, and we denote by $\gapsepmax = \max_{\hat{y}, y, a} \gapsep{\hat{y}, y, a}$. In particular, \textit{separation} holds if and only if $\gapsepmax = 0$. A second criterion called \emph{sufficiency} refers to the conditional
independence of $Y$ and $A$ given $\hat{Y}$, i.e. $Y \indep A| \hat{Y}$. As in
the case of \textit{separation}, we can define $\gapsuff{y, \hat{y}, a} = |\Pr[Y = y |
\hat{Y}=\hat{y}, A= a] - \Pr[Y=y | \hat{Y}=\hat{y}]|$, and denote by
$\gapsuffmax = \max_{y, \hat{y}, a} \gapsuff{y,\hat{y}, a}$. \textit{Sufficiency} holds if and only if $\gapsuffmax = 0$. 

To measure the performance of the classifiers, we can measure the
\emph{precision} and \emph{recall} of the positive and negative classes. For
the positive class, these are denoted by $\PPV$ (positive predictive value) and $\TPR$ (true positive rate), whereas for the
negative class these are denote by $\NPV$ (negative predictive value) and $\TNR$ (true negative rate).%
%
%
The complements of these notions are also meaningful. In particular, we have (see \cite{tharwat2020classification}):
\begin{itemize}
	\item $\PPV = \Pr[Y = 1 | \hat{Y}=1]$, $\FDR = 1 - \PPV = \Pr[Y = 0 | \hat{Y}=1]$ (false discovery rate)
	\item $\TPR = \Pr[\hat{Y} = 1 | Y=1]$, $\FNR = 1 - \TPR = \Pr[\hat{Y}=0 | Y = 1]$ (false negative rate)
	\item $\NPV = \Pr[Y = 0 | \hat{Y}=0]$, $\FOR = \Pr[Y=1 | \hat{Y}=0]$ (false omission rate)
	\item $\TNR = \Pr[\hat{Y} = 0 | Y=0]$, $\FPR = \Pr[\hat{Y}=1 | Y=0]$ (false positive rate)
\end{itemize}

We also consider these measures conditioned on the sensitive attribute, e.g. $\PPV_a = \Pr[Y = 1 | \hat{Y} = 1, A = a]$; the other measures are defined similarly.

We say that a set of values $\{a_1, \ldots, a_n \}$ are $(\eps,
\delta)$-approximately equal if for every $i, j$, we have that $a_i \leq a_j
e^{\eps} + \delta$. We say that $(\hat{Y}, Y, A)$ satisfies $(\epsilon,
\delta)$-\textit{separation}, if the four sets $\left\{ \Pr\left[\hat{Y} = \hat{y} | Y=y
, A=a\right] | a \in \{0, 1\}\right\}$ of two values for each possible setting of
$y, \hat{y} \in \{0, 1\}$ are $(\eps, \delta)$-approximately equal. Likewise,
we say that $(\hat{Y}, Y, A)$ satisfies $(\eps, \delta)$-\textit{sufficiency}, if the
four sets $\left\{\Pr\left[Y=y|\hat{Y}=\hat{y}, A=a\right] | a \in\{0, 1\}
\right\}$ of two values for each possible setting of $y, \hat{y} \in \{0, 1\}$
are $(\eps, \delta)$-approximately equal. We denote by $\rho = \Pr[Y = 1]/\Pr[Y =0]$ and by $\rho_a = \Pr[Y = 1  | A
= a]/\Pr[Y = 0 | A= a]$. Note that division by 0 will not occur due to the assumptions, as stated in Thm.~\ref{thm:main}. We refer to $\rho$ (resp. $\rho_a$) as the
\emph{base odds} (resp. group conditional base odds). 

\subsection{Theoretical Result: Refining the Impossibility Result}
As mentioned in the introduction, several authors have shown impossibility
results for the \textit{separation} and \textit{sufficiency} criteria.
\citeauthor{kleinberg-et-al:LIPIcs:2017:8156} further proved an approximate version of the impossibility result, but at heart of it was showing that the following quantity $\gamma = \TPR \cdot \PPV + \FNR \cdot \FOR$ is close to $1$ whenever the base rates differ substantially and approximate versions of both calibration and \textit{sufficiency} hold. Moreover, $\gamma$ is 1 if and only if 
the learned classifier is perfect ($Y = \hat{Y}$) or perfectly flipped ($Y = 1 - \hat{Y}$).  However, the interpretation is more difficult in the approximate case.

\paragraph{How does \citeauthor{kleinberg-et-al:LIPIcs:2017:8156}'s $\gamma$ differ from our notion of accuracy?}\label{par:ext_significance}

 
 A concrete example of the difference between our result and \cite{kleinberg-et-al:LIPIcs:2017:8156} is the following. Consider a classifier achieving the following performances: number of true negatives = 1, number of false positives = 1, number of false negatives = N, number of true positives = 1, then $\gamma$ is very close to 1 (essentially larger than $(\frac{N}{N+1})^2$), while PPV = 0.5, leading to a large gap between $\gamma$ and the actual performance of the classifier. On the other hand, one can come up with an example where the inverse phenomenon happens: $\gamma$ is close to 0.5 while PPV is close to 1. This is arguably equally bad: It is precisely a phenomenon that we want to prevent in e.g., a criminal justice setting (such as described by COMPAS). Motivated by practice, we would like to show that the accuracy is close to 1 whenever the base rates differ substantially and approximate versions of both calibration and \textit{sufficiency} hold, meaning that a predictor with some level of accuracy could satisfy simultaneously some minimal approximate calibration and \textit{sufficiency}. We thus present the following theorem.



\begin{thm} 
\label{thm:main}
	Assume that for some fixed constant $c \in (0, 1/2)$, the random variable
	$(\hat{Y}, Y, A)$ satisfies for all $a, y, \hat{y} \in \{0, 1\}$, $\Pr[Y = y
	| A =a], \Pr[\hat{Y} = \hat{y} | A = a] \in (c, 1-c)$.  
	Let $\eps > 0$ be sufficiently small, in particular $\eps < c$ and $0 < \delta = o(\eps)$.
	If $(\hat{Y}, Y, A)$ satisfies
	both $(\eps, \delta)$-\textit{sufficiency} and \textit{separation}, then at least one of the following holds: 
	\begin{enumerate}
		\item For each $a$ and $\mu \in \{\PPV, \NPV, \TPR, \TNR\}$, we have $\mu_a \geq 1 - O(\eps + \delta/\eps)$,
\item For each $a$ and $\mu \in \{\PPV, \NPV, \TPR, \TNR\}$, we have $\mu_a \leq  O(\eps + \delta/\eps)$ (flipped classifiers),
		\item The set $\{\rho_0, \rho_1 \}$ is $(O(\eps + \delta/\eps), 0)$-approximately equal, where $\rho_a = \Pr[Y = 1  | A = a]/\Pr[Y = 0 | A = a]$.
	\end{enumerate}
\end{thm}

We note that the second property is unavoidable. A classifier with accuracy $0$, i.e. $\hat{Y} = 1 - Y$ giving a flipped classifier satisfies \textit{separation} and \textit{sufficiency}. The following straightforward observation also allows us to re-frame the result in more standard notion of accuracy, i.e. $\Pr[\hat{Y} \neq Y]$.

\begin{observation}\label{obs:accuracy}
Assume that for all $a$, $\TPR_a \geq 1- \epsilon$ and $\TNR_a \geq 1-\epsilon$, then the accuracy $\mathrm{acc} = \Pr[\hat{Y} \neq Y] \geq 1 - \epsilon$. Likewise, if  for all $a$, $\TPR_a \leq \epsilon$ and $\TNR_a \leq \epsilon$, then the accuracy $\mathrm{acc} \leq \epsilon$. 
\end{observation}

The proof of Theorem~\ref{thm:main} needs careful analysis, but is fundamentally based on simple relationships between the measures. The full proof appears in Appendix~\ref{sec:theoretical_result}.
However, we provide an overview of the key idea. Assuming that the quantities under consideration are well defined, i.e. no division by $0$ occurs, one can easily establish relationships between the measures defined above 
using Bayes' rule. For e.g., as shown in \cite{chouldechova2017fair}, we get,
\begin{align*}
    \FPR_a = \rho_a \cdot \frac{\FDR_a}{\PPV_a} \cdot \TPR_a.
\end{align*}

If \textit{separation} and \textit{sufficiency} were to hold exactly, all of these quantities are the same regardless of the group $a$. The only way this can happen is if the base odds, $\rho_a$ are the same, or $\FPR_a = 0$ for all $a$. Similar observations about other quantities shows that we must have a perfect classifier. These equations as written down do not hold if some quantities, e.g. $\PPV_a$ are $0$. A careful analysis shows that that actually happens when we have a perfectly flipped classifier, i.e. $\hat{Y} = 1 - Y$.

%% file: sections/4.Method.tex
Following our theoretical result, a natural next step would be to consider an optimization procedure for classifiers that takes advantage of this rediscovered relationship between $\Delta$ violations and accuracy, as viewed through $\PPV, \NPV, \TPR$ and $\TNR$. 
Thus, we suggest to simply minimize the upper bound of $\gapsuff{}{}$ and $\gapsep{}{}$, as a possible regularizer added to classification accuracy, an objective in a multi-objective framework, or a finetuning objective:
\begin{align}
    \mathcal{L}(\theta, Y, A) =\allgapmax
\end{align}
$\hat{Y}$ for the purpose of $\Delta^{notion}_{c, a}$ is defined as $S = \theta^T X > \tau$, and $\tau = 0.5$ is our default value for classification thresholding.

Given that we study the binary group case, in the following we implemented the objective $\allgapmax$ as the max between $\gapsuff{y, \hat{y}} = |\Pr[Y=y | \hat{Y}=\hat{y}, A=1] - \Pr[Y=y | \hat{Y}=\hat{y}, A=0]|$ and $\gapsep{y, \hat{y}} = |\Pr[Y=y | \hat{Y}=\hat{y}, A=1] - \Pr[Y=y | \hat{Y}=\hat{y}, A=0]|$. We name this modification $\allgapmaxbinarydiff$. Alternating between the two formulations of the objectives leads only to slight differences, as can be seen in results included in the appendix.

\textbf{Using $\allgapmax$ as part of a multi-objective approach.} The balancing act we aim to achieve is compatible with many works published in recent years supporting fairness constraints, such as Multi-Objective techniques. 
Although multiple such methods were suggested recently \cite{Celis_10.1145/3287560.3287586, Ruchte_DBLP:journals/corr/abs-2103-13392}, we chose to focus on the work by \citeauthor{pmlr-v161-padh21a}. MAMO, the algorithm developed in \cite{pmlr-v161-padh21a} is model-agnostic, proposes a novel hyperbolic tangent-based relaxation to fairness criteria, and is accompanied by an easy-to-use implementation that is modular and allows for adaptation and extensions. We test the performance of our $\allgapmax$ as an objective to be optimized with MAMO. 




\textbf{Finetuning approach.}\label{sec:finetuning_obj}
 We also consider $\allgapmax$ as an addition to the common binary cross-entropy (BCE) optimization, making it a finetuning objective. 
 It can easily follow common model training procedures, saving in costs and implementation difficulties, effectively being a post-processing step of trained models. One point of inspiration for this approach is provided by existing work that shows unconstrained accuracy-focused optimization of predictors tend to minimize $\gapsuff{}$ as a natural by-product \citep{liu2019implicit}. Therefore, we propose to finetune a predictor trained with BCE, which is expected to already minimize $\gapsuff{}$, by forcing the upper bound over $\allgapmax$ to be lower. Crucially, it can act as a safeguard against too sharp of an increase in one $\Delta$ violation in response to applying a fairness constraint, e.g. just minimizing $\gapsep{}$ may lead to a high increase in $\gapsuff{}$ and we would like to upper bound such increases.

%% file: sections/7.Results.tex

\begin{table*}
\centering
\resizebox{\textwidth}{!}{%
\begin{tabular}{@{}|l|lllllll|@{}}
\toprule
\rowcolor[HTML]{CFCDCD} 
\cellcolor[HTML]{EFEFEF}Dataset &
  \multicolumn{1}{l|}{\cellcolor[HTML]{CFCFCF}Number of samples} &
  \multicolumn{1}{l|}{\cellcolor[HTML]{CFCFCF}Training set size} &
  \multicolumn{1}{l|}{\cellcolor[HTML]{CFCDCD}Validation/Test set size} &
  \multicolumn{1}{l|}{\cellcolor[HTML]{CFCDCD}\begin{tabular}[c]{@{}l@{}}P(Y=0), \\ P(Y=1)\end{tabular}} &
  \multicolumn{1}{l|}{\cellcolor[HTML]{CFCDCD}\begin{tabular}[c]{@{}l@{}}P(A=0), \\ P(A=1)\end{tabular}} &
  \multicolumn{1}{l|}{\cellcolor[HTML]{CFCDCD}\begin{tabular}[c]{@{}l@{}}P(Y=1 | A=0), \\ P(Y=1 | A=1)\end{tabular}} &
  $\gamma_0$, $\gamma_1$ \\ \midrule
 \cellcolor[HTML]{CFCDCD}\texttt{Color-MNIST} &
  {\color[HTML]{000000} 60,000} &
  {\color[HTML]{000000} 40,000} &
  20,000 &
  0.49, 0.51 &
  0.51, 0.49 &
  0.69, 0.29 &
  2.25, 0.41 \\ \cmidrule(r){1-1}
\cellcolor[HTML]{CFCDCD}\texttt{COMPAS} &
  {\color[HTML]{000000} 6,172} &
  {\color[HTML]{000000} 4,115} &
  2,057 &
  0.53, 0.47 &
  0.6, 0.4 &
  0.53, 0.39 &
  1.1, 0.64 \\ \cmidrule(r){1-1}
\cellcolor[HTML]{CFCDCD}\texttt{NELS} &
  {\color[HTML]{000000} 4,743} &
  {\color[HTML]{000000} 3,162} &
  1,581 &
  0.533, 0.467 &
  0.099, 0.901 &
  0.311, 0.484 &
  0.45, 0.94 \\ \cmidrule(r){1-1}
\cellcolor[HTML]{CFCDCD}\texttt{Adult Income} &
  {\color[HTML]{000000} 48,842} &
  {\color[HTML]{000000} 32,562} &
  16,280 &
  0.76, 0.239 &
  0.67, 0.33 &
  0.3, 0.11 &
  0.12, 0.44 \\  \bottomrule
\end{tabular}%
}
\vspace{-0.05cm}
\caption{
Description of datasets. 
\vspace{-0.15cm}
}
\label{tab:datasets}
\end{table*}

\begin{figure*}[t!]
     \centering
     \subfloat[\texttt{Color-MNIST}]{
     \begin{minipage}[c]{.24\textwidth}
         \centering
         \includegraphics[width=\textwidth]{imgs/mnist_maxfull_radar_plt.png}
     \end{minipage}
    }
     \subfloat[\texttt{COMPAS}]{
     \begin{minipage}[c]{.24\textwidth}
         \centering
         \includegraphics[width=\textwidth]{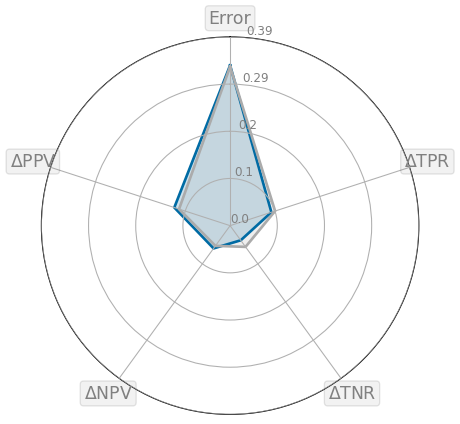}
     \end{minipage}}
    \subfloat[\texttt{NELS}]{
     \begin{minipage}[c]{.24\textwidth}
         \centering
         \includegraphics[width=\textwidth]{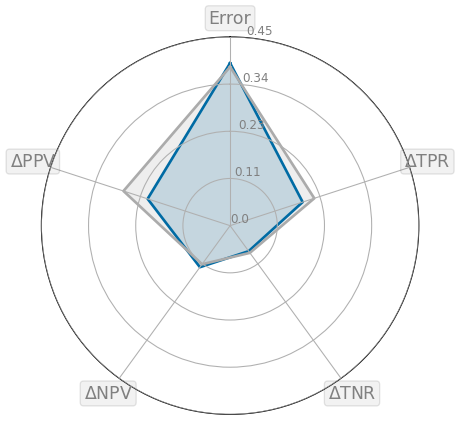}
     \end{minipage}
    }
    \subfloat[\texttt{Adult Income}]{
     \begin{minipage}[c]{.24\textwidth}
         \centering
         \includegraphics[width=\textwidth]{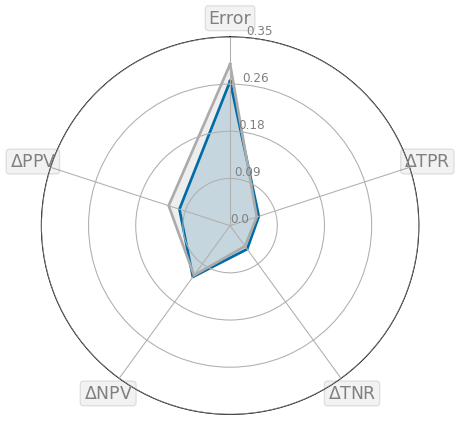}
     \end{minipage}
    }
      \caption{Error and fairness violations in terms of $\PPV$, $\NPV$, $\FPR$ and $\TPR$ absolute difference across groups for DEO \cite{Hardt_NIPS2016_9d268236} vs. $\allgapmaxbinarydiff$ (indicated as \textbf{Ours} in the legend above). Values plotted are the mean of each metric for 40 runs of all datasets, except for \texttt{Color-MNIST}, which is averaged over 25 runs. We look for a centered shape and smaller area where applicable.
    \vspace{-0.5cm}
      } \label{fig:MAMO_Min_Max}
\end{figure*}

All experiments were completed on a MacBook Pro 2019 using the Pytorch library for Python, and could be easily performed on any modern CPU configuration, aside for the \texttt{Color-MNIST} experiments which involved training a Convolutional Neural Network on a GPU. In all the following plots, $\Delta$s defined as, e.g. $\text{TPR} = |\Pr[\hat{Y}=1| A=1, Y=1] - \Pr[\hat{Y}=1| A=0, Y=1]|$).

\subsection{Datasets}\label{sec:experiments_datasets}

See descriptive statistics in Table \ref{tab:datasets}. All are available under creative commons license CC$0$. We elaborate on \texttt{Color-MNIST} as it is the only dataset that we curate for fair ML experiments. We defer exact details of the rest to the appendix and will make datasets and processing code available. 
\textbf{Color-MNIST.} The first dataset we considered is a tweaked-version of the original MNIST dataset \cite{Lecun1998_726791}. In order to clearly demonstrate the dynamics the theory points to, this setting enables us to study a quality predictor that can achieve high accuracy, and thus also small fairness violations. At the same time, it also involves a larger dataset of the more challenging image modality, compared to the often tabular data studied in fairness contexts, and is thus closer to certain real-world applications. 
We create a version of the original dataset where images are colored based on their label. First, we binarize the label, to make all images labeled as $0-4$ the negative class, and all others, $5-9$ the positive class. Next, all images are assigned a color such that, initially, the images in the positive class are assigned a red color, and all images in the negative class are assigned a green color. However, with probability 0.3, we flip the color assigned to each image. We use the assigned colors as the sensitive attribute,
and by the procedure above we end up with a dataset that includes different base-rates for each label-color combination, such that each sensitive group is more associated with a different label (see Table \ref{tab:datasets}). Similar procedures were described in \cite{arjovsky2020invariant, wang2020fairness} for different purposes.

 \begin{figure*}[t!]
 \centering
 \subfloat[\texttt{Color-MNIST}]{
     \begin{minipage}[c]{.242\textwidth}
         \centering
        \includegraphics[width=\textwidth]{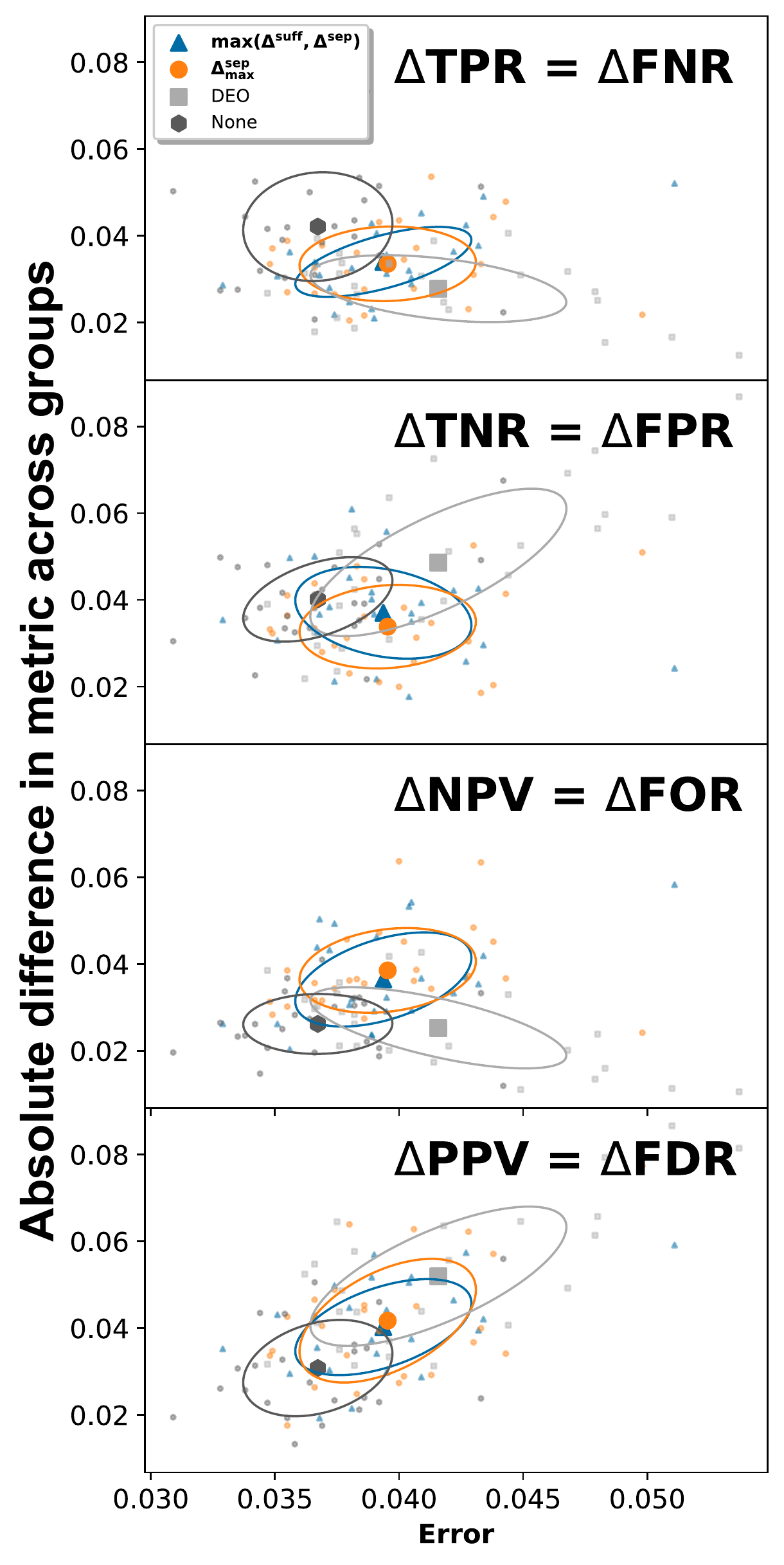}
     \end{minipage}}
     \subfloat[\texttt{COMPAS}]{
     \begin{minipage}[c]{.22\textwidth}
         \centering
        \includegraphics[width=\textwidth]{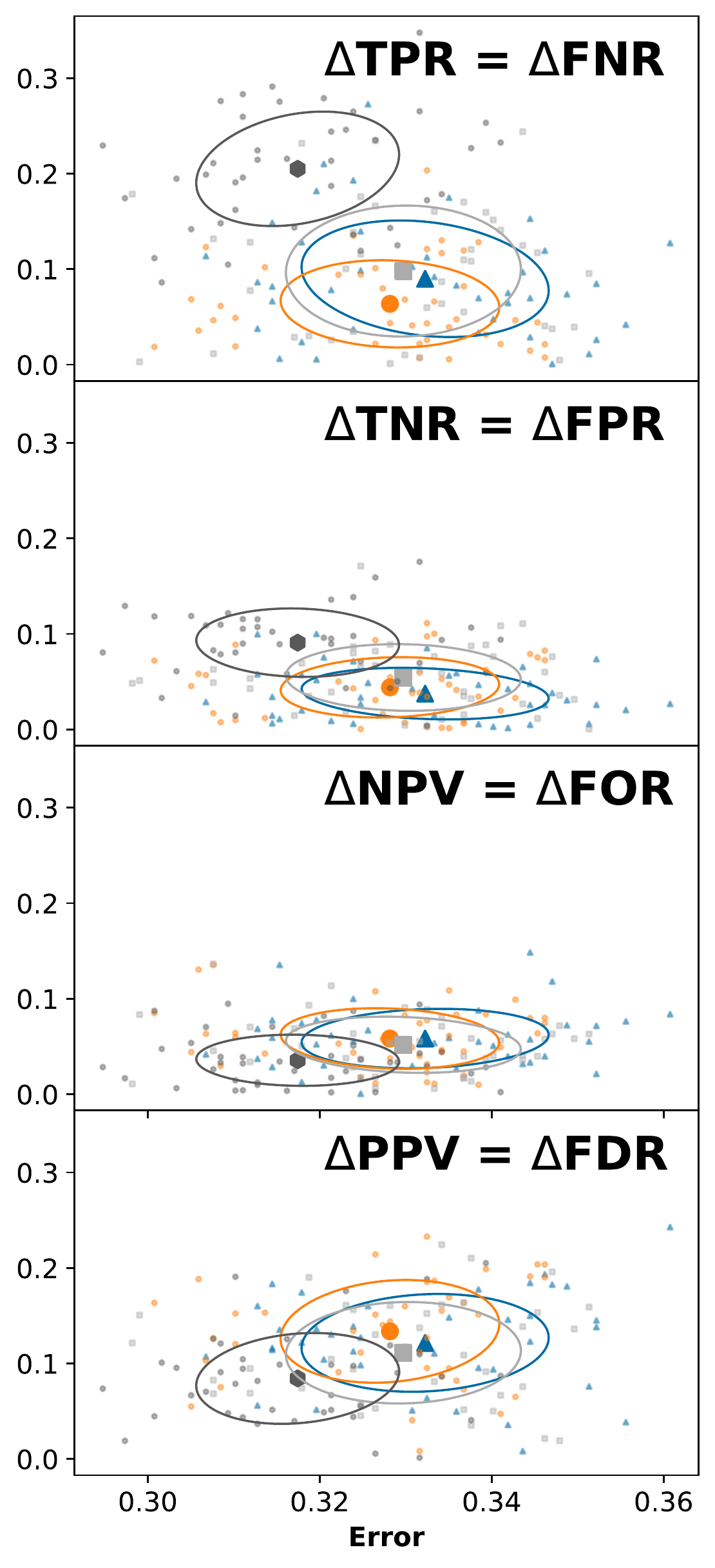}
     \end{minipage}}
    \subfloat[\texttt{NELS}]{
     \begin{minipage}[c]{.22\textwidth}
         \centering
        \includegraphics[width=\textwidth]{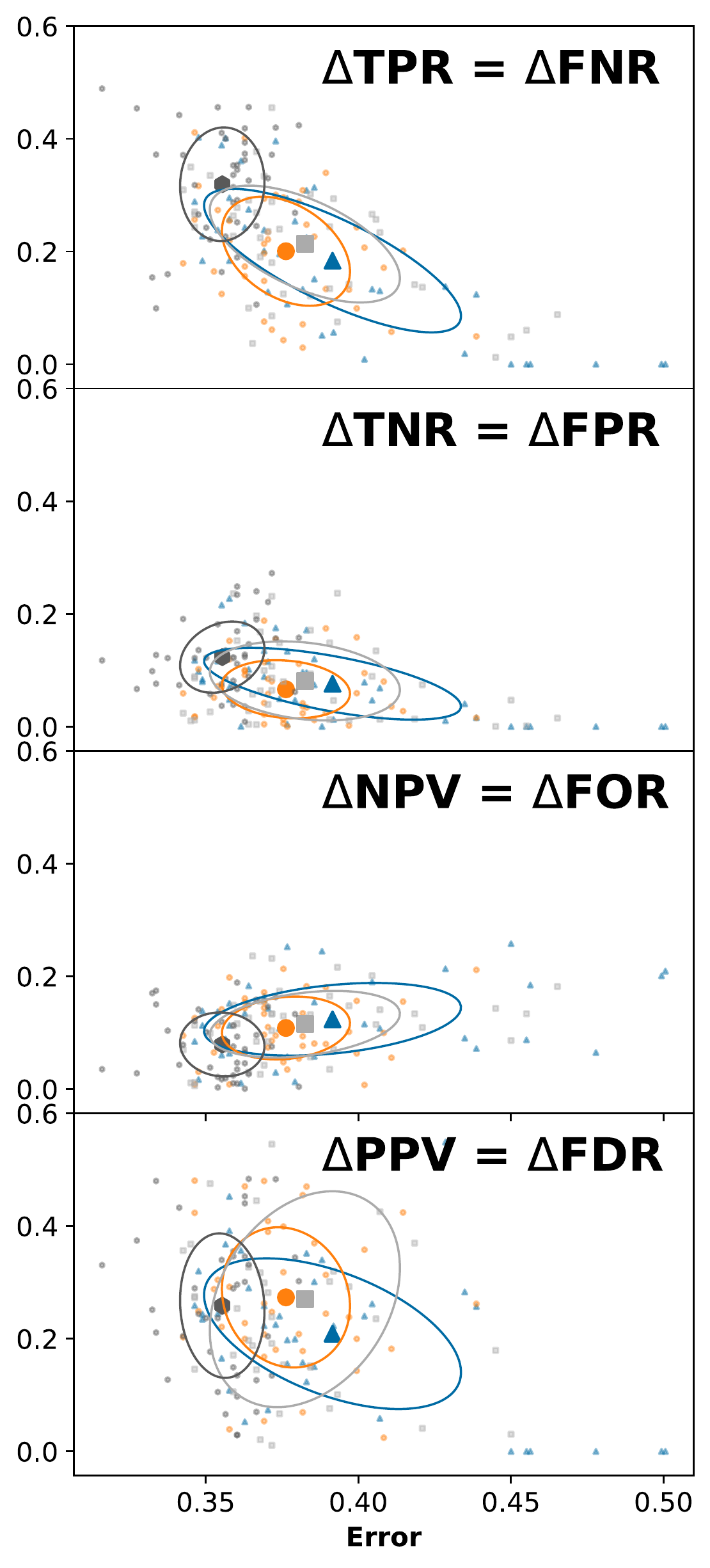}
     \end{minipage}
    }
    \subfloat[\texttt{Adult Income}]{
     \begin{minipage}[c]{.22\textwidth}
         \centering
        \includegraphics[width=\textwidth]{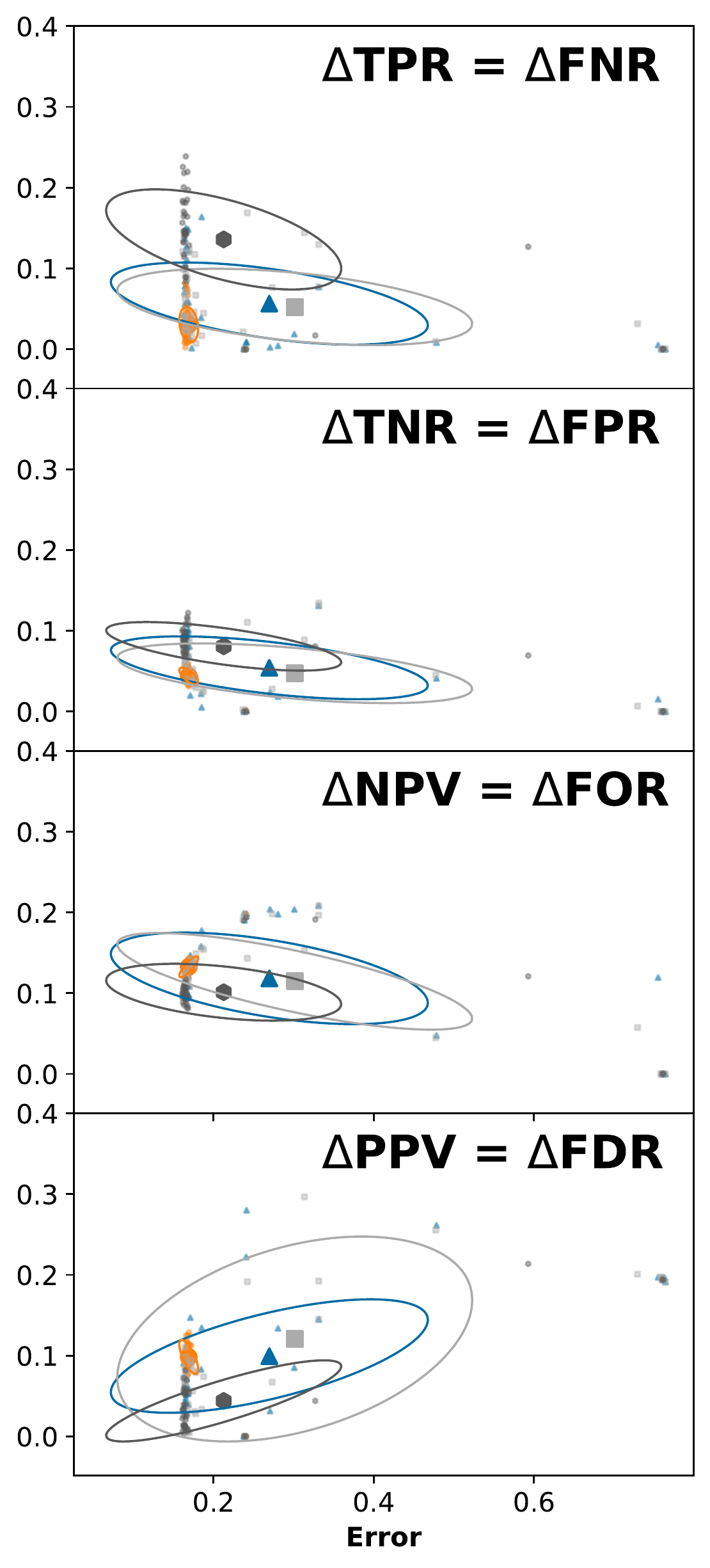}
     \end{minipage}
    }
    \caption{Multi-Objective results, using DEO, $\allgapmaxbinarydiff$ or $\gapsepmax$ as an additional objective to Binary-Cross Entropy, implemented via \citeauthor{pmlr-v161-padh21a}'s framework, MAMO. 
    Location of shapes in the middle of ellipses is the mean of 40 runs (25 runs for \texttt{Color-MNIST}). 
    The first 2 rows correspond to \textit{separation} measures; the bottom 2 to \textit{sufficiency} measures. 
    'None' refers to the unconstrained case (training with Binary Cross-Entropy loss). Ellipses correspond to 1-std. of the distribution of results over runs.  y-axis: absolute difference of metric across groups; x-axis: error (1-accuracy).}   
    \vspace{-0.2cm}
    \label{fig:MAMO_1col} 
\end{figure*}


\subsection{Results Multi-Objective} \label{sec:MAMO_results} 
 As one can see in Figures~\ref{fig:MAMO_Min_Max} and \ref{fig:MAMO_1col}, $\allgapmaxbinarydiff$ can help achieve a better balance across all $\Delta$ violations compared to DEO as well as other possible constraints or an unconstrained approach. We directly compare in this case to the setting and results presented in \cite{pmlr-v161-padh21a} (See Figure 2 in their manuscript. We change their error bars to ellipses, visualizing how those errors are correlated). We adapt their code, and keep most of their original hyperparameter choice for \texttt{COMPAS} and \texttt{Adult Income}. We average results over 40 runs, and set the regularizer coefficient $\lambda$ to 0.1 and 0.3 respectively. For \texttt{Color-MNIST} we use 10 training epochs, and set the hyperparameter $\lambda$'s value to 2. In all settings, we achieve lower violations in at least 2 out of the 4 possible violations, without trading much accuracy or much increase in the other measures. We sacrifice relatively little in accuracy compared to the unconstrained case; we also report better or equivalent accuracy compared to DEO. We further experimented with a variety of alternatives 
 but found $\allgapmaxbinarydiff$ to lead to better trade-offs. We present results for 
 alternative criteria in the appendix, and include below the 3 most promising constraints: $\allgapmaxbinarydiff$ $\gapsepmax$ and DEO. 

\textbf{Color-MNIST.} \underline{DEO vs. $\allgapmaxbinarydiff$.} As expected, DEO achieves lower $\TPR$/$\FNR$\footnote{The DEO constraint directly involves the absolute difference between $\TPR$/$\FNR$ for both groups as indicated by the sensitive attribute. In the following, whenever we refer to the constraint itself we will call it DEO, but when referring to the measured violation in $\TPR$/$\FNR$ we will simply refer to $\TPR$/$\FNR$.} and $\NPV$/$\FOR$ values (by about 0.01 points), but at the cost of higher $\PPV$/$\FDR$ values (\textit{sufficiency}) (0.012 points difference, and 0.02 compared to the unconstrained model’s violation) as well as $\TNR$/$\FPR$ values (\textit{separation}) (0.0116 difference from $\allgapmaxbinarydiff$).
 In other words, DEO achieves high fluctuations of two pairs of violations. On the other hand, $\allgapmaxbinarydiff$ and $\gapsepmax$ keep all 4 violations within a similar range, between 0.033 and 0.04, thus balancing the violations across all measures and keeping accuracy rather close to the unconstrained model training (0.037 vs. 0.039). \underline{$\allgapmaxbinarydiff$ vs. $\gapsepmax$.} $\allgapmaxbinarydiff$ and $\gapsepmax$ stay close together in this dataset (up to 0.003 difference). However, notice how $\gapsepmax$ leads to slightly lower \textit{separation} values in this case, while $\allgapmaxbinarydiff$ achieves slightly lower \textit{sufficiency} values. Thus we see $\allgapmaxbinarydiff$ achieving a better balance across the 4 violation pairs. \underline{The variance across runs} is much higher for the DEO objective. $\allgapmaxbinarydiff$ and $\gapsepmax$ are within the same range.

\textbf{COMPAS.} \underline{DEO vs. $\allgapmaxbinarydiff$.} Notice that $\TPR$/$\FNR$ is the biggest original violation for the unconstrained model ($\sim$ 0.2). Thus, DEO and $\allgapmaxbinarydiff$ achieve similar results for that criterion (0.098 vs. 0.09). However, $\allgapmaxbinarydiff$ additionally improves group-difference in $\TNR$/$\FPR$ (0.054 vs. 0.037), without much sacrifice of increase of group-differences in $\PPV$/$\FDR$ (0.11 vs. 0.12). Considering Figure~\ref{fig:MAMO_1col} more closely, we see $\allgapmaxbinarydiff$ offers lower violation for all \textit{separation} measures (top two rows), while not trading off much of the \textit{sufficiency} violations. \underline{$\gapsepmax$ vs. $\allgapmaxbinarydiff$} Compared to $\gapsepmax$, we see similar \textit{separation} values to $\allgapmaxbinarydiff$ ($\TPR$/$\FNR$, $\TNR$/$\FPR$). Yet, $\gapsepmax$ achieves smaller $\TPR$/$\FNR$ values compared to $\allgapmaxbinarydiff$ (0.06 vs. 0.089), but does so at the prices of slightly higher $\TNR$/$\FPR$ violation (0.044 vs. 0.037), and more importantly higher violations in the \textit{sufficiency} violations $\PPV$/$\FDR$ (0.134 vs. 0.122). \underline{The variance across runs} is overall similar for the 3 constraint methods. Thus, we see $\allgapmaxbinarydiff$ as achieving a better balance overall, although $\gapsepmax$ could be a reasonable alternative for specific use cases.

\textbf{NELS.} \underline{DEO vs. $\allgapmaxbinarydiff$.} Both approaches attain similar results for $\TNR$/$\FPR$ (0.081 vs. 0.076) and $\NPV$/$\FOR$ (0.115 vs. 0.124), however, there is a slight advantage for $\allgapmaxbinarydiff$ for $\TPR$/$\FNR$ 
and especially for $\PPV$/$\FDR$ (0.27 vs. 0.209). \underline{$\gapsepmax$ vs. $\allgapmaxbinarydiff$.}  The two objectives show similar results for $\TNR$/$\FPR$ 
and $\NPV$/$\FOR$ 
with slightly smaller values for $\gapsepmax$. Nevertheless, there is an even greater advantage for $\allgapmaxbinarydiff$ for $\TPR$/$\FNR$ 
and especially for $\PPV$/$\FDR$ (0.273 vs. 0.209). \underline{Variance across runs} is quite high for $\allgapmaxbinarydiff$ for this setting, yet seems to be dominated by 3 runs out of 40 (which dominate the scale of the plot as a whole), representing runs that are both high error and particularly high $\PPV$/$\FDR$ violations.

\begin{figure*}[t!]
     \centering
     \subfloat[\texttt{Color-MNIST}]{
     \begin{minipage}[c]{.24\textwidth}
         \centering
         \includegraphics[width=\textwidth]{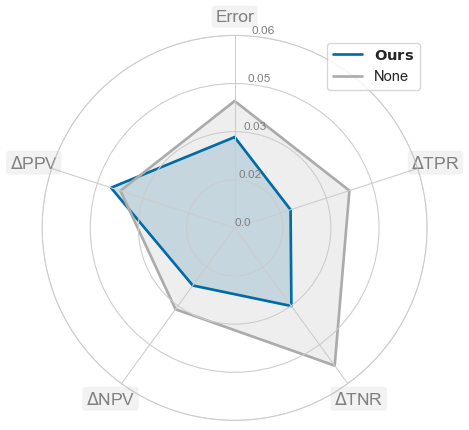}
     \end{minipage}
    }
     \subfloat[\texttt{COMPAS}]{
     \begin{minipage}[c]{.24\textwidth}
         \centering
         \includegraphics[width=\textwidth]{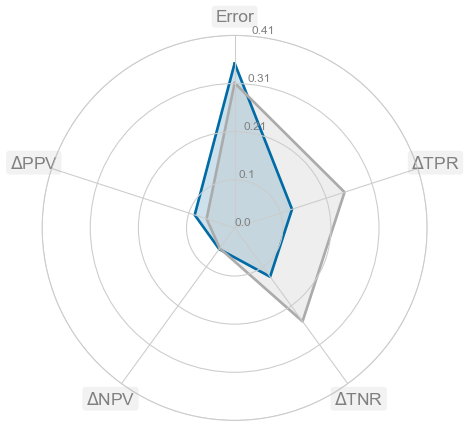}
     \end{minipage}}
    \subfloat[\texttt{NELS}]{
     \begin{minipage}[c]{.24\textwidth}
         \centering
         \includegraphics[width=\textwidth]{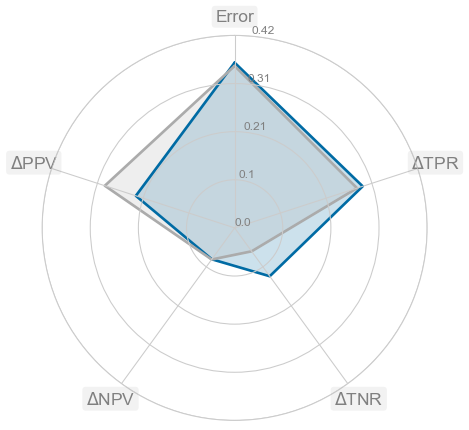}
     \end{minipage}
    }

      \caption{Finetune experiment with $\allgapmaxbinarydiff$ objective. Error and fairness violations in terms of $\PPV$, $\NPV$, $\FPR$ and $\TPR$ absolute difference across groups for BCE vs. $\allgapmaxbinarydiff$ (indicated as \textbf{Ours} in the legend above). We look for greater 'balance' of shape around the center, and smaller area where applicable.
      \vspace{-0.5cm}
      } \label{fig:fntn_binarydiff_full}
\end{figure*}

\textbf{Adult Income.} \underline{DEO and $\allgapmaxbinarydiff$.} \textit{Separation} values in this case remain quite close.
As for \textit{sufficiency}, DEO achieves slightly lower $\NPV$/$\FOR$ (0.114 vs. 0.118); notice, however, that $\allgapmaxbinarydiff$ achieves far lower $\PPV$/$\FDR$ violation (0.12 vs. 0.099), and that’s with an even lower variance in performance. \underline{$\gapsepmax$ and $\allgapmaxbinarydiff$.} It is rather clear $\gapsepmax$ leads to lower \textit{separation} values, while $\allgapmaxbinarydiff$ achieves one far lower \textit{sufficiency} violation.  The accuracy achieved by $\gapsepmax$ is higher here, and so $\gapsepmax$ might be a better choice for this dataset overall. \underline{Variance across runs.} For $\allgapmaxbinarydiff$ and DEO the variance is quite large for the \texttt{Adult Income} dataset, with a slightly higher variance for DEO. $\gapsepmax$ achieves remarkably low variance, and also a high accuracy that is very competitive with the unconstrained case. Thus, it could be an advantage for this criterion, depending on the specifics of the use case. However, two things are worth noting: the variance of $\allgapmaxbinarydiff$ and DEO seem to be driven by a few runs that achieve particularly high error. The \texttt{Adult Income} dataset is highly imbalanced in terms of label distribution (the ratio of negative to positive labels is 3.16, see Table \ref{tab:datasets}), which might explain why it can 
be vulnerable to different random splits of the dataset into train, validation and test.

\subsection{Results Finetuning}\label{sec:fntn_results}
As a second possibility of putting our objective to the test, we consider using $\allgapmaxbinarydiff$ as part of a finetuning approach, following standard Binary Cross Entropy training. We test the performance of such an objective on multiple datasets as we did above (see main results achieved on the test set in 
Figure~\ref{fig:fntn_binarydiff_full}),
but include only the discussion for the \texttt{COMPAS} dataset in the main text; see the rest of the discussion in the Appendix. Notice how $\allgapmaxbinarydiff$ leads to improved mean violations of a greater scale than small increase in error, or even improvement for the \texttt{Color-MNIST} dataset. These results are overall in line with the performance of the Multi-Objective approach of the previous section. Given the performance of $\allgapmaxbinarydiff$ above, we focus on it as the objective used in this finetuning setting, and compare its results to the baseline of training a predictor via BCE throughout training. Similarly, given the high variance and imbalanced nature of the \texttt{Adult Income} dataset we did not include it in the following. 
\textbf{COMPAS.} We applied the method discussed in Section \ref{sec:finetuning_obj} to the \texttt{COMPAS} dataset to train a predictor. In the previous experiment, we mostly adopted the hyper-parameter settings used in \cite{pmlr-v161-padh21a} for ease of comparison. In this setting, we performed our own hyperparameter tuning, as described in the Appendix. The chosen settings will also be included in the accompanying code. 
We see similar training dynamics in terms of accuracy and $\allgapmaxbinarydiff$ on both training and validation sets (the indicated validation results are the mean over the 3 folds and the 5 random initializations, see Figure~\ref{fig:COMPAS_fntn_binarydiff_curves} in the Appendix). We can see some trade-off between high accuracy and low $\allgapmaxbinarydiff$, yet \texttt{finetune} remains well above 60\% accuracy while achieving a decrease of about 40\% in $\allgapmaxbinarydiff$ on the train/validation splits. 
Inspecting $\allgapmaxbinarydiff$ and accuracy results over the test set in Figure~\ref{fig:fntn_binarydiff_full}, we see that our proposed method is competitive with BCE. It presents the better balanced violations achieved by $\allgapmaxbinarydiff$, captured by the more centered and smaller area it covers in the radar plot compared to the unconstrained BCE, while the top edge representing the error remains quite close. 

%% file: sections/6.Conclusion.tex
We have demonstrated that beyond the impossibility result emanating from the COMPAS debate, there is a possibility as well. While one cannot hold perfect calibration and still have equal error rates in settings where a perfect classifier is not available and labels are not distributed equally across groups, the approximate versions of \textit{sufficiency} and \textit{separation} allow us to look for better balance within the well-known trade-off. 
We have shown how the $\allgapmax$ criterion can be used within a Multi-Objective setting as well as a finetuning regime. We emphasize that this is one possible approach to minimizing both \textit{sufficiency} and \textit{separation} violations, but one could prioritize other forms of trade-offs, e.g., designing objectives that put greater emphasis on one prediction class or type of violation, based on the use case. However, we aimed to demonstrate the advantage of trying to optimize both in the general case. 


%% file: sections/A.Appendix.tex
In what follows, we include additional discussion, proofs of our theoretical results, as well as additional details about our experiments. 

\section{How is the classic impossibility result cited and used in the fairness literature?}
\cite{Pleiss_NIPS2017_b8b9c74a} showed the strongest resistance to an attempt to balance \textit{separation} and \textit{sufficiency} violations, stating multiple times that a practitioner may want to focus on one. ``Calibration and error rate constraints are in most cases mutually incompatible goals'', \citeauthor{Pleiss_NIPS2017_b8b9c74a} advised, and proceeded to conclude, ``[i]n practical settings, it may be advisable to choose only one of these goals rather than attempting to achieve some relaxed notion of both''. Other works used similar language to justify different fairness approaches, including \cite{kusner2017counterfactual} who recommended Counterfactual Fairness as a way to bypass the impossibility result on multiple occasions; \cite{pmlr-v161-padh21a} recently discussed performance of their mult-objective approach to fairness in terms of the impossibility result, explaining that ``For the COMPAS dataset, M-MAMO-fair performs well for DEO but not for DDP...  it is not possible to satisfy DP and error rate based metrics simultaneously if the base rate of classification is different for different groups... This explains the poor performance of the M-MAMO-fair algorithm on the Dutch dataset as well as the fact that it performs well only on DEO and not on DDP for the COMPAS dataset''. We advocate here both a closer look at the fairness violation resulting from any fairness constraint (see our Figure~\ref{fig:MAMO_1col} compared to Figure 2 in the original), as well as an additional criterion aiming at achieving a balance between \textit{sufficiency} and \textit{separation} measures, instead of abandoning the hope to find such balance altogether.

\section{Proofs from Section~\ref{sec:theory_results}}

\subsection{Proof of Theorem~\ref{thm:main}}\label{sec:theoretical_result}
\begin{proof}[Proof of Theorem~\ref{thm:main}]
	Throughout the proof, we assume that $c_1, c_2, \ldots$ are sufficiently
	small constants. We will prove the statements for $\TPR$ and $\PPV$; the
	statements regarding $\NPV, \TNR$ hold similarly. Without loss of
	generality assume that $\FPR_0 \leq \FPR_1$. We have,
	\begin{align*}
		\FPR_0 &= \frac{\Pr[\hat{Y} = 1, Y=0 | A =0]}{\Pr[Y = 0 | A=0]}.
	\end{align*}

	We break the proof into several cases. 

	\textbf{Case 1:} $\FPR_0 \leq c_1 \epsilon$. In this case, $\FPR_1 \leq c_1 \epsilon e^\eps + \delta$. For $a \in \{0, 1\}$, we have that 
	\begin{align*}
		\PPV_a &= \frac{\Pr[\hat{Y}=1,Y=1|A = a]}{\Pr[\hat{Y}=1 | A=a]}= 1 - \frac{\Pr[ \hat{Y}=1,Y=0|A = a]}{\Pr[\hat{Y}=1 | A=a]}  = 1 - \frac{\Pr[Y=0 | A=a]}{\Pr[\hat{Y}=1 | A=a]} \FPR_a.
	\end{align*}
	From there, using the assumption that $\Pr[Y=0|A =a], \Pr[\hat{Y}=1 | A=a]
	\in (c, 1-c)$,  and the facts that $e^\epsilon = \Theta(1)$ and $\delta =
	O(\delta/\epsilon)$, it follows that $\PPV_a \geq 1 - O(\eps +
	\delta/\eps)$. Since $\FPR_0 \leq c_1 \epsilon$, we have that
	$\Pr[\hat{Y}=1, Y=0 | A=0] \leq (1 - c)\cdot c_1 \epsilon$. Then, provided
	$c_1$ is small enough, e.g. $c_1 \leq c/2$, we have $\Pr[\hat{Y} = 0, Y = 0
	| A] \geq c/2$. we have, by Bayes' rule, 
	\begin{align*}
		1 - \TPR_a = \FNR_a &=  \frac{\Pr[\hat{Y}=0, Y=1 | A=a]}{\Pr[Y=1 | A=a]} \\
		&=  \frac{\Pr[Y=1 | \hat{Y}=0, A=a] \cdot \Pr[\hat{Y}=0| A=a]}{\rho_a \cdot  \Pr[Y=0 | A=a] }\\
				&=  \frac{\Pr[Y=1 | \hat{Y}=0, A=a] }{\rho_a \cdot  \Pr[Y=0 | A=a] } \cdot
				\frac{\Pr[\hat{Y}=0 | Y=0, A=a] \cdot \Pr[Y=0 | A=a]}{\Pr[{Y}=0 | \hat{Y}=0, A=a]} \\
								  &= \Pr[\hat{Y}= 0 | Y=0, A=a] \cdot \frac{1}{\rho_a} \cdot \frac{\Pr[Y =1 | \hat{Y}=0, A=a ]}{\Pr[Y=0 | \hat{Y}=0, A=a]} \\
								  &= \TNR_a \cdot \frac{1}{\rho_a} \cdot \frac{\FOR_a}{\NPV_a}.
								  \intertext{The above can be rewritten slightly to get,}
		\rho_a &= \frac{\TNR_a}{\FNR_a} \cdot \frac{\FOR_a}{\NPV_a}.
	\end{align*}
	If $\FNR_a \leq c_2 \epsilon$, for either $a = 0$ or $a = 1$, then it 
	clearly follows that $\TPR_a \geq 1 - O( \epsilon + \delta/\epsilon)$ for $a \in \{0,
	1\}$. Otherwise, it must be that case that $\Pr[\hat{Y}=0, Y = 1 | A = a]
	\geq c_2 c \eps$. Together with $\Pr[\hat{Y}=0, Y= 0 | A= a] \geq c/2$, this
	means that each of $\TNR_a$, $\FNR_a$, $\FOR_a$, and $\NPV_a$ are at least
	$c_3 \eps$ for some constant $c_3$. 
	Since $(\hat{Y}, Y, A)$ satisfies
	both $(\eps, \delta)$-\textit{sufficiency} and \textit{separation} we get
 for every $q\in \{ \TNR, \FOR, \FNR, \NPV\}$ and $a' \neq a, a' \in \{0,1\}$ we have that
	$\{q_a, q_{a'}\}$ are ($\eps,\delta$)-approximately equal.
	Hence, by Claim~\ref{claim:Pferd},
	\[e^{-(\eps + (\delta/(c\eps)) } \leq q_a/q_{a'} \leq e^{(\eps + (\delta/(c\eps))}.
 \]
	Therefore, using the claim below it follows that $\{\rho_0, \rho_1\}$ are
	$(O(\eps + \delta/\eps), 0)$-approximately equal.

	\begin{claim}\label{claim:Pferd} If $\{v_1, v_2\}$ are $(\eps, \delta)$-approximately equal and $\min\{v_1, v_2\} \geq c \eps$, we have $e^{-(\eps + (\delta/(c\eps)) } \leq v_1/v_2 \leq e^{(\eps + (\delta/(c\eps))}$.
	\end{claim}
	\begin{proof}
		Without loss of generality, let $v_1 \leq v_2$. Then we have,
		\begin{align*}
			v_1 \leq v_2 &\leq v_1 e^\eps + \delta  \leq v_1 e^\eps + v_1 \cdot \frac{\delta}{ c\eps} \leq v_1 \cdot e^{\eps + \frac{\delta}{c \eps}},
		\end{align*}
		where we used that $1+x/e^\eps\leq 1+x \leq e^x$.\end{proof}

	\textbf{Case 2a}: Assume that $\FPR_0 \geq c_1 \eps$; further assume that
	$\Pr[\hat{Y}=1, Y=1 | A=0] \leq c_4 \epsilon$. The latter assumption implies
	that $\PPV_0 \leq c_5 \epsilon$ and $\TPR_0 \leq c_6 \epsilon$. Using the fact that $\{\PPV_0, \PPV_1\}$  are $(\eps, \delta)$-approximately equal and $\{\TPR_0, \TPR_1\}$ are $(\eps, \delta)$-approximately equal, we get the required result. 

	\textbf{Case 2b}: Assume that $\FPR_0 \geq c_1 \eps$; further assume that
	$\Pr[\hat{Y}=1, Y=1 | A=0] \geq c_4 \epsilon$. Then, as in Case 1, we can write,
	\begin{align*}
		\FPR_a &= \frac{\Pr[\hat{Y}=1, Y=0 | A=a]}{\Pr[Y=0 | A=a]}  \\
				 &= \Pr[\hat{Y}=1 | Y=1, A=a] \cdot \rho_a \cdot \frac{\Pr[Y=0 | \hat{Y}=1, A=a]}{\Pr[Y=1| \hat{Y}=1, A=a]} \\
				 &= \TPR_a \cdot \rho_a \cdot \frac{\FDR_a}{\PPV_a}.
				 \intertext{From the above, we get,}
		\rho_a &= \frac{\FPR_a}{\TPR_a} \cdot \frac{\PPV_a}{\FDR_a}.
	\end{align*}

	Under the conditions, each of $\TPR_a, \FPR_a, \PPV_a, \FDR_a$ is at least some $c_7 \eps$, which using the claim above implies that $\{\rho_0, \rho_1\}$ are $(O(\eps  + \delta/\eps), 0)$-approximately equal.
\end{proof}

\subsection{Proof of Observation~\ref{obs:accuracy}}
\begin{proof}[Proof of Observation~\ref{obs:accuracy}]
We have
\begin{align*}
	\mathrm{acc} &= \Pr[\hat Y= Y] = \sum_a \left(\Pr[\hat Y=0, Y = 0, A=a] + \Pr[\hat Y=1, Y = 1, A=a]\right) \\
&
= \sum_a \left(\TNR_a \cdot \Pr[Y=0,A=a]+ \TPR_a \cdot \Pr[Y=1,A=a] \right)  \geq 1-\eps. 
\end{align*}
The proof of the second part follows almost identically. 
\end{proof}

\section{Additional Experimental Details}

\subsection{Dataset Information}
 \paragraph{COMPAS.} A natural interesting choice for a dataset to investigate this approach is the one that ignited the debate to begin with: the \texttt{COMPAS} dataset as studied by ProPublica \citep{larson2016we}\footnote{We use the \textsc{compas scores two years} dataset from \url{https://github.com/propublica/compas-analysis}.}. We apply minimal pre-processing, and finally obtain a dataset of 9 features, including \textsc{jail in, jail out, age, prior counts, days b screening arrest, charge degree, race, age category} and \textsc{sex}. We use \textsc{race} as the protected attribute $A$, and \textsc{two year recidivism} as the target $Y$. For the purpose of this analysis, we only consider samples where $A$ was either \textsc{African-American} or \textsc{Caucasian}. The dataset is overall balanced in terms of group sizes, but with unequal base rates, as can be seen in Table \ref{tab:datasets}. 

\paragraph{NELS.} As an additional dataset we consider the \texttt{NELS} dataset, or the US
Department of Education’s National Education
Longitudinal Study of 1988 \citep{ingels1990national} as well as its followups. The dataset was used in \citep{kleinberg2018algorithmic}, and we follow the pre-processing steps taken by the authors in the accompanying code\footnote{The original code is at \url{https://www.openicpsr.org/openicpsr/project/114435/version/V1/view.}}. It consists of 427 features, including the protected attribute \textsc{race}, as well as \textsc{high-school grades, course taking patterns, extracurriculars}, and \textsc{standardized tests in Math, Reading, Science, Social Studies}. The target for prediction is \textsc{gpa}. The only change we make to \citeauthor{kleinberg2018algorithmic}'s procedure is setting the threshold for binarizing the target \textsc{gpa} slightly higher: while \citeauthor{kleinberg2018algorithmic} consider setting the threshold at ``At least mostly B's received'' we set it at ``At least A's and B's received''\footnote{A and B are used in this context as grades in American educational institutions; another way to view the difference is setting the threshold for $Y=1$ at $>3.25$ instead of $>2.75$ in \citeauthor{kleinberg2018algorithmic}.}. The resulting dataset is overall balanced in terms of the prediction classes, but unbalanced in base rate and highly unbalanced group sizes. See Table \ref{tab:datasets}.

\paragraph{Adult Income.} We adopt the \texttt{Adult Income} dataset \cite{Kohavi_10.5555/3001460.3001502} for the purpose of a fairness application from \cite{pmlr-v161-padh21a}, and thus follow their pre-processing steps. The resulting dataset consists of 48,842 samples, and 49 features, with a combination of categorical and continuous variables. The target for prediction is a binarized \textsc{Income} variable, and the sensitive attribute in this case is \textsc{sex}\footnote{\citeauthor{pmlr-v161-padh21a} also allow for using \textsc{race} as an additional sensitive attribute, but we focus on the single group setting in this work.}. We note the \texttt{Adult Income} dataset shows an imbalance in class labels, as it includes many more negative examples than positive examples. While \cite{pmlr-v161-padh21a} do not adjust their learning strategies or their reporting to this setting, we note that a constant predictor would achieve an accuracy of 76\% by simply predicting the negative label for all examples. Thus, for the finetuning approach we consider, we employed a correction for imbalance during training (see Pytorch's BCE documentation). Finally, we would like to acknowledge our familiarity with recent work suggesting to stop using this dataset for fairness studies \cite{Ding_DBLP:journals/corr/abs-2108-04884}, but we would like to emphasize we include it here to allow for direct comparison with \cite{pmlr-v161-padh21a}. 

\subsection{Data Preprocessing}
We discuss the generation and treatment of \texttt{Color-MNIST} in the main text. 


For the \texttt{COMPAS} and \texttt{NELS} datasets, we mainly followed the pre-processing steps taken by the original authors and analyzers \citep{larson2016we, kleinberg2018algorithmic}. 


Finally, we adopt the treatment of the \texttt{Adult Income} dataset directly from \cite{pmlr-v161-padh21a}. 


\subsection{Data Splits and Cross Validation}
For all datsets, two-thirds of the dataset were used as a training set. As for validation and test splits, those were created and selected via kfold cross validation on the remaining third of each dataset. Notice that there is a slight difference in pipelines between the multi-objective experiments, where we tried to stay as close as possible to the procedures in \cite{pmlr-v161-padh21a} for comparison purposes, while for the finetuning objective experiments we developed our own pipeline. Thus, for details on the exact procedure for validation and averaging of results over runs for the multi-objective case, see \cite{pmlr-v161-padh21a}. 

For the finetuning experiments, we used 3-folds for cross-validation for \texttt{COMPAS} and \texttt{MNIST} (as the model was more complex and required a little longer training time), while for NELS we used 5-folds. Those were operationalized via the scikit-learn \citep{pedregosa2011scikit} library, using StratifiedKFold in particular. Stratification was done over the sensitive attribute, to make sure the same proportion of the minority group is achieved in each fold. 


\begin{figure}[!ht]
     \centering
     \subfloat[\texttt{COMPAS}]{
     \begin{minipage}[c]{0.465\textwidth}
         \centering
         \includegraphics[width=\textwidth]{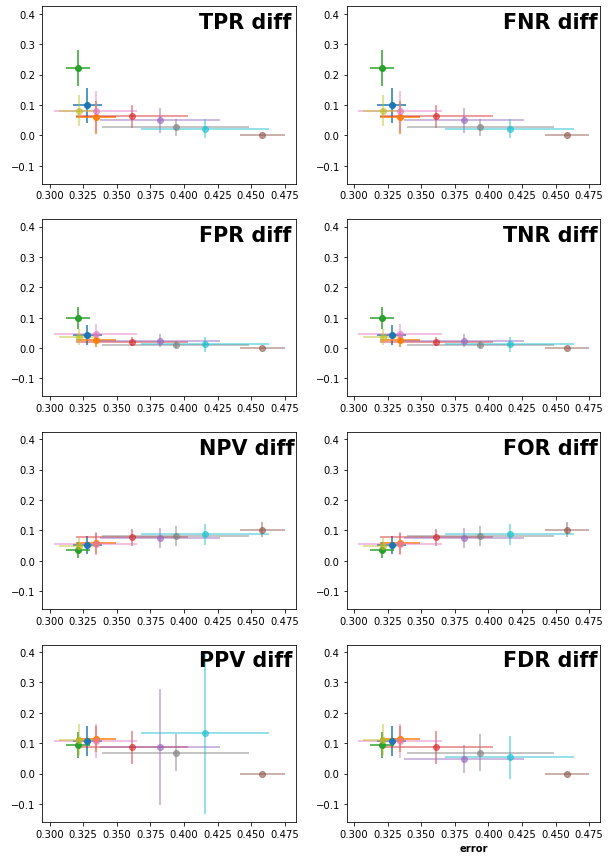} \label{fig:COMPAS_all8_w_alt_objs}
     \end{minipage}}
     \subfloat[\texttt{Adult}]{
     \begin{minipage}[c]{0.46\textwidth}
         \centering
         \includegraphics[width=\textwidth]{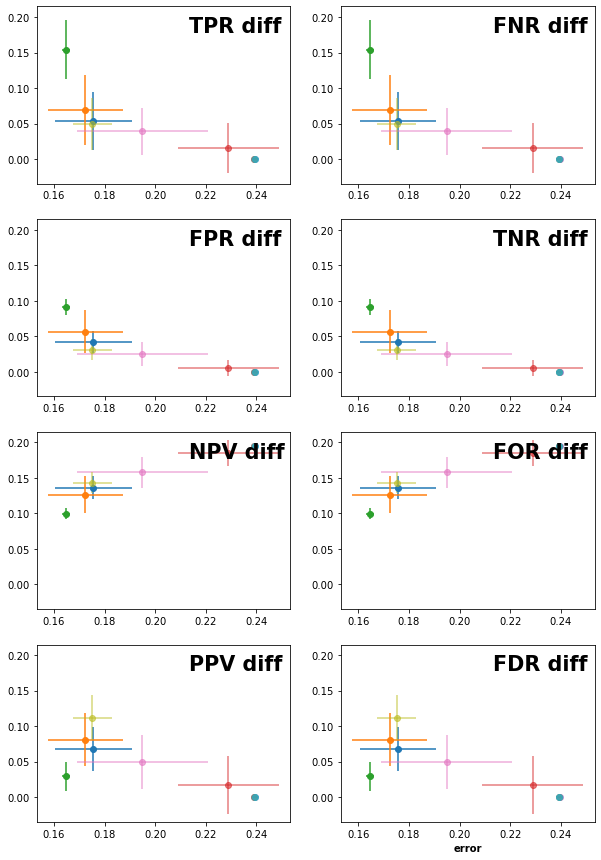} \label{fig:Adult_all8_w_alt_objs}
     \end{minipage}}
    \subfloat{
     \begin{minipage}[c]{0.1\textwidth}
         \centering
         \includegraphics[width=\textwidth]{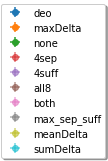} 
     \end{minipage}}
    \caption{An example of performance of alternative objectives on the COMPAS dataset, averaged over 25 on the test set. The main point of this figure is to demonstrate we did consider alternative objectives, as the multi-objective framework indeed allowed, but chose to focus on the most promising forms we ended up including in the main text.}
    \label{fig:all8_alt_objs}
\end{figure}

\section{Additional experimental results}

\subsection{Multi-Objective results}

On top of the results reported in the main text, we have also considered other forms of objectives to represent different forms of constraints. Among these were the sum of all $\Delta$-violations, the mean of all $\Delta$-violations, all 4 $\gapsep{}$ (4sep), all 4 $\gapsuff{}$ (4suff) or both (all8), both $\gapsepmax$ and $\gapsuffmax$ (called max in the plot below), as well as both DEO and DP (exactly like in \cite{pmlr-v161-padh21a}, and called both in the following figure's legend).

In Figure~\ref{fig:all8_alt_objs} we provide an example of such objectives applied to the \texttt{COMPAS} and \texttt{Adult} datasets. The pattern of performance was somewhat similar across our trials, and we decided to focus on the objective presented in the main text instead.

\subsection{Finetuning results}

\begin{figure*}[t!]
     \centering
     \subfloat[\texttt{Color-MNIST}]{
     \begin{minipage}[c]{.24\textwidth}
         \centering
         \includegraphics[width=\textwidth]{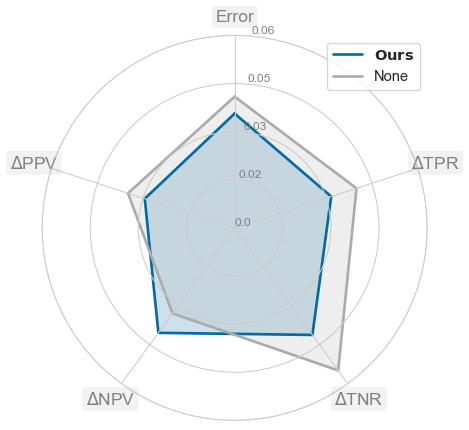}
     \end{minipage}
    }
     \subfloat[\texttt{COMPAS}]{
     \begin{minipage}[c]{.24\textwidth}
         \centering
         \includegraphics[width=\textwidth]{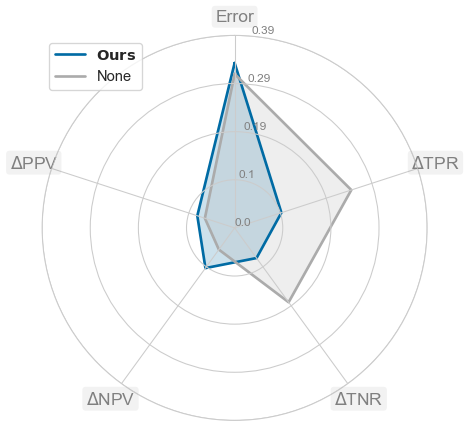}
     \end{minipage}}
    \subfloat[\texttt{NELS}]{
     \begin{minipage}[c]{.24\textwidth}
         \centering
         \includegraphics[width=\textwidth]{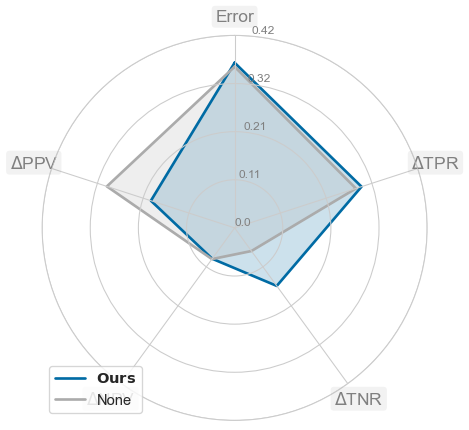}
     \end{minipage}
    }

      \caption{Finetune experiment with $\allgapmax$ objective (instead of $\allgapmaxbinarydiff$ used in practice in the main text). Error and fairness violations in terms of $\PPV$, $\NPV$, $\FPR$ and $\TPR$ absolute difference across groups for BCE vs. $\allgapmax$ (indicated as \textbf{Ours} in the legend above). 
      } \label{fig:fntn_maxDelta_full}
\end{figure*}

\subsubsection{Hyperparmeter Choice for Fine-tuning Experiments}
A range of hyperparameters were considered for obtaining the results in the manuscript. We will attach the code to reproduce the results alongside this supplemental, to make results easy to reproduce, and hopefully make all choices we made explicit. We discuss the gist of them here.

Values for each of these hyperparameters, for each one of \texttt{Color-MNIST}, \texttt{COMPAS} and \texttt{NELS}, were chosen via grid search. The ranges we considered for each one of the key hyperparameters are as follows:

\begin{itemize}
    \item $\gamma \in [0.1, 0.99]$ - controls the decrease in learning rate at each step taken by the scheduler. 
    \item \textbf{weight decay} $\in [1e^{-6}, 1e^{-1}]$ - strength of penalty on magnitude of model parameters, akin to $\ell_2$ regularization.
    \item \textbf{scheduler step size} $\in [20, 200]$ - determines the intervals of epochs at which the scheduler applies a reduction in learning rate.
    \item \textbf{learning rate} $\in [1e^{-8}, 1e^{-1}]$ - learning rate used by optimizer to update model parameters.
    \item \textbf{hidden dimensions} $\in [64, 512]$ - number of hidden dimensions of the Multi-Layered Perceptron model (\texttt{NELS} uses a single-layered logistic regression, so not applicable).
    \item \textbf{finetune learning rate} $\in [1e^{-8}, 1e^{-1}]$ - learning rate used once objective function changed as part of the finetuning regime.
    \item \textbf{finetune $\gamma$} $\in [0.1, 0.85]$ - $\gamma$ (as above) used once objective function changed as part of the finetuning regime. 
\end{itemize}


\begin{figure*}[ht!]
     \centering
     \subfloat[Training split]{
     \begin{minipage}[c]{.42\textwidth}
         \centering
         \includegraphics[width=.47\textwidth]{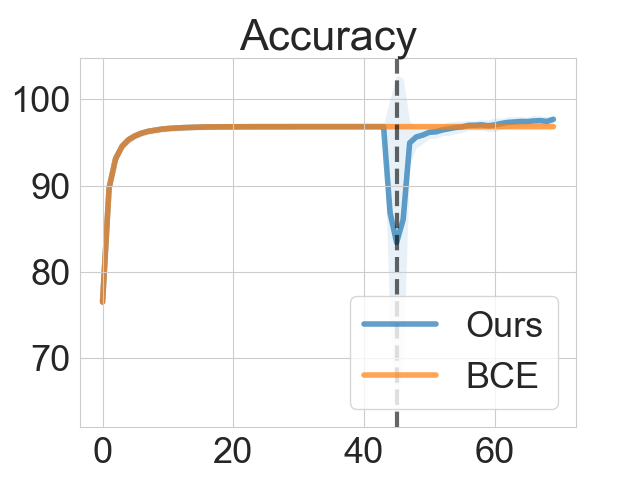}
         \includegraphics[width=.47\textwidth]{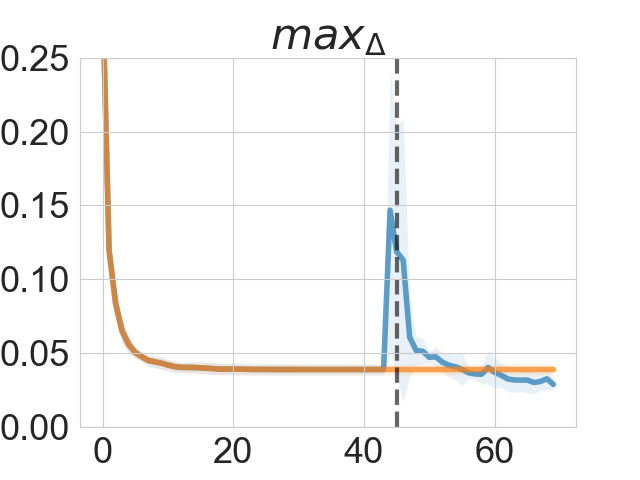}
         \label{fig:MNIST_fntn_binarydiff_train}
     \end{minipage}}
    \subfloat[Validation split]{
     \begin{minipage}[c]{.42\textwidth}
         \centering
         \includegraphics[width=.47\textwidth]{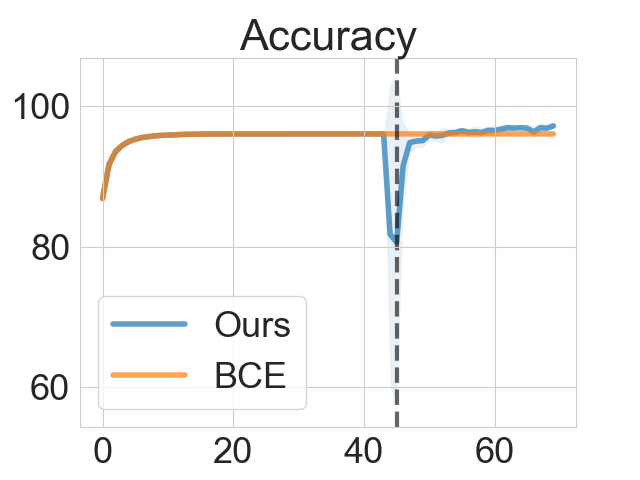}
         \includegraphics[width=.47\textwidth]{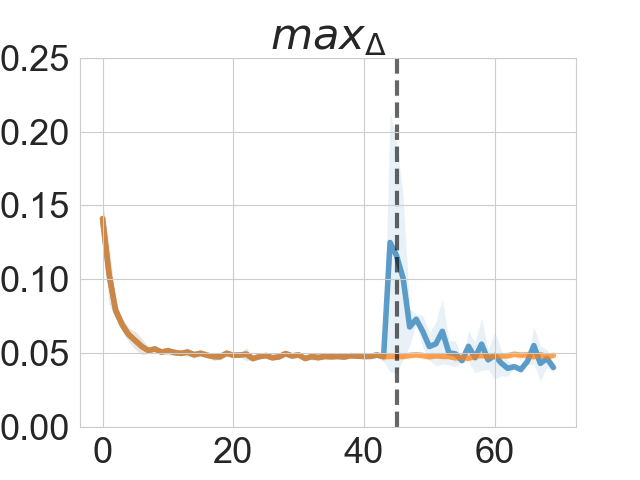}
         \label{fig:MNIST_fntn_binarydiff_validation}
     \end{minipage}}
        \caption{\texttt{Color-MNIST} \texttt{finetune} $\allgapmaxbinarydiff$ experiment. Curves indicate the mean and 95\%-confidence intervals for 3 different model initializations. Validation results are also averaged over 3 cross-validation (CV) folds. The dashed line indicates the loss switch for the finetuning (here BCE to $\allgapmaxbinarydiff$). 
        }
        \label{fig:MNIST_fntn_binarydiff_curves}
\end{figure*}

\begin{figure*}[ht!]
     \centering
     \subfloat[Training split]{
     \begin{minipage}[c]{.42\textwidth}
         \centering
         \includegraphics[width=.47\textwidth]{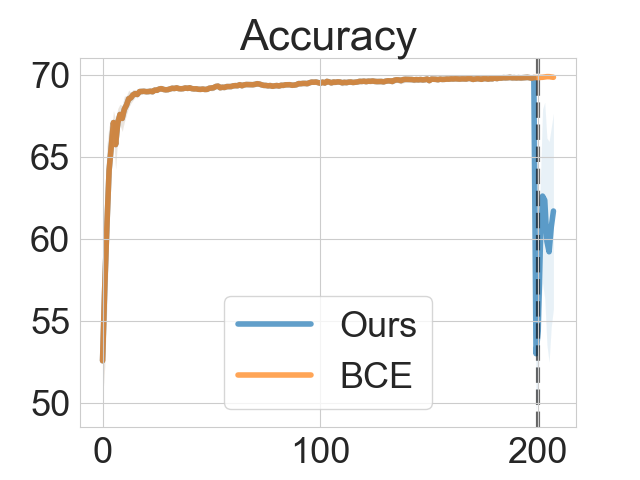}
         \includegraphics[width=.47\textwidth]{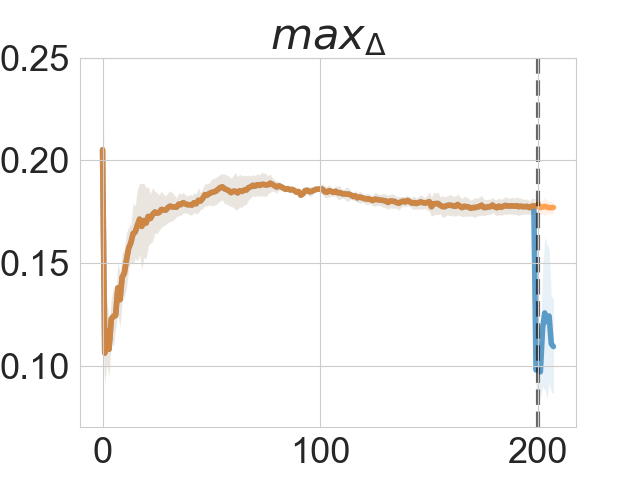}
         \label{fig:COMPAS_fntn_binarydiff_train}
     \end{minipage}}
    \subfloat[Validation split]{
     \begin{minipage}[c]{.42\textwidth}
         \centering
         \includegraphics[width=.47\textwidth]{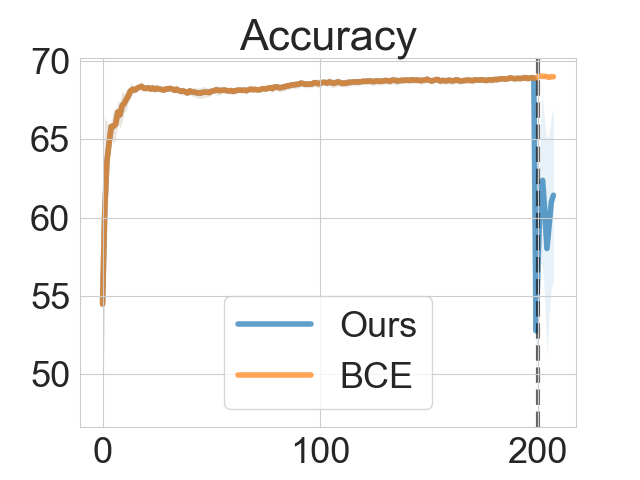}
         \includegraphics[width=.47\textwidth]{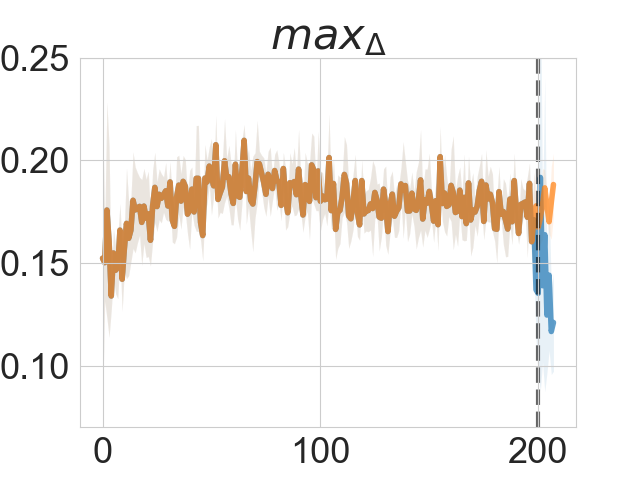}
         \label{fig:COMPAS_fntn_binarydiff_validation}
     \end{minipage}}
        \caption{\texttt{COMPAS} \texttt{finetune} $\allgapmaxbinarydiff$ experiment. Curves indicate the mean and 95\%-confidence intervals for 5 different model initializations. Validation results are also averaged over 3 cross-validation (CV) folds. The dashed line indicates the loss switch for the finetuning (here BCE to $\allgapmaxbinarydiff$). 
        }
        \label{fig:COMPAS_fntn_binarydiff_curves}
\end{figure*}

\begin{figure*}[ht!]
     \centering
     \subfloat[Training split]{
     \begin{minipage}[c]{.42\textwidth}
         \centering
         \includegraphics[width=.47\textwidth]{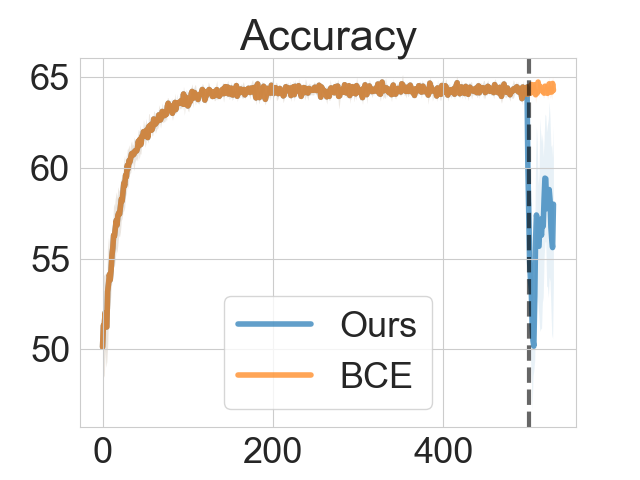}
         \includegraphics[width=.47\textwidth]{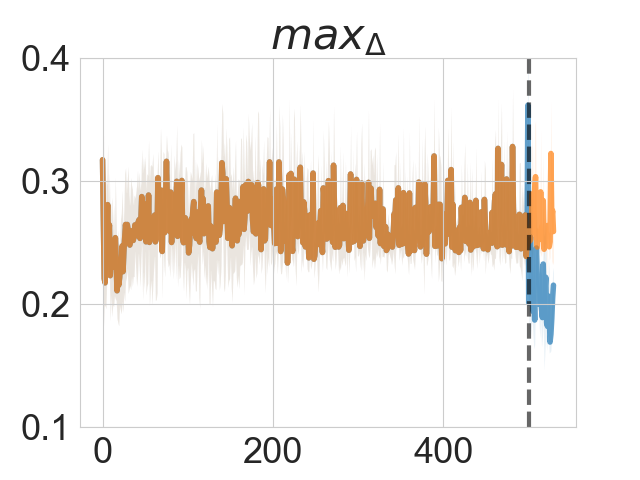}
         \label{fig:NELS_fntn_binarydiff_train}
     \end{minipage}}
    \subfloat[Validation split]{
     \begin{minipage}[c]{.42\textwidth}
         \centering
         \includegraphics[width=.47\textwidth]{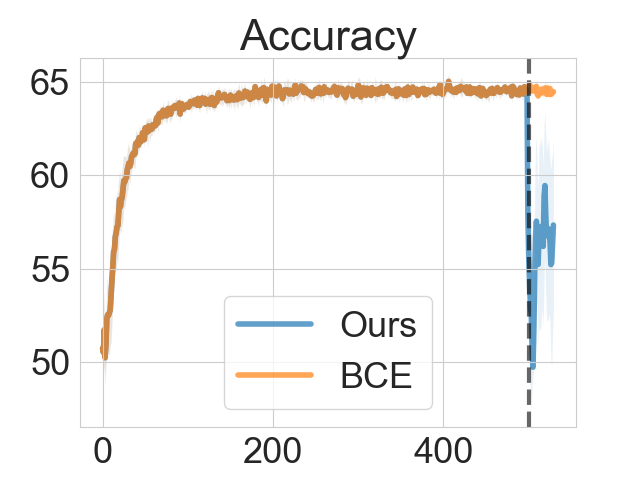}
         \includegraphics[width=.47\textwidth]{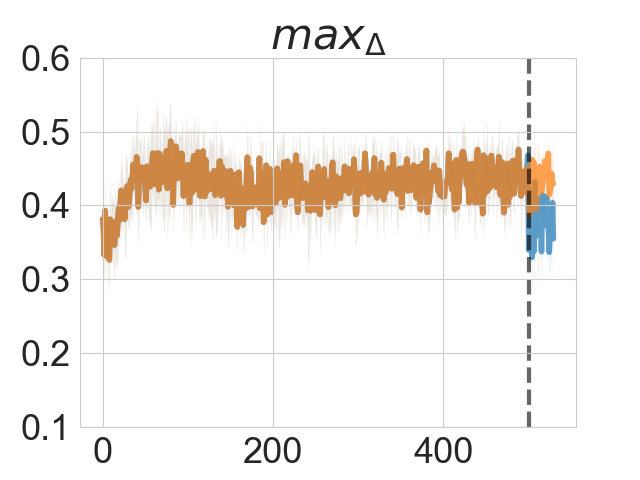}
         \label{fig:NELS_fntn_binarydiff_validation}
     \end{minipage}}
        \caption{\texttt{NELS} \texttt{finetune} $\allgapmaxbinarydiff$ experiment. Curves indicate the mean and 95\%-confidence intervals for 5 different model initializations. Validation results are also averaged over 5 cross-validation (CV) folds. The dashed line indicates the loss switch for the finetuning (here BCE to $\allgapmaxbinarydiff$). 
        }
        \label{fig:NELS_fntn_binarydiff_curves}
\end{figure*}

\subsubsection{Additional Experiments}
In the main text, we demonstrated the applicability of a finetuning optimization with our proposed objectives. We focused on the \texttt{COMPAS} dataset to provide a proof of concept, but elaborate here on the rest of the results that make up figure~\ref{fig:fntn_binarydiff_full} in the main text. 
In Figure~\ref{fig:fntn_maxDelta_full}, we additionally present results for using the finetuning approach with the $\allgapmax$ variant of our criteria, rather than $\allgapmaxbinarydiff$. We do so to argue that one could use either to achieve similar goals. Hyperparameter choices are overall comparable for both of these objective formulations, with an occasional difference in number of epochs run or a slight difference in finetune learning rate.


\paragraph{Color-MNIST.} Similarly to the behavior we have seen for \texttt{COMPAS} in the main text, in Figures~\ref{fig:fntn_binarydiff_full} (main text) and \ref{fig:fntn_maxDelta_full} we see we are able to fine better balancing point between fairness violations in terms of absolute group difference in PPV and NPV (\textit{sufficiency} violations) and TPR and FPR (\textit{separation} violations). Notice how we are able to find better accuracy predictors, that are also better in terms of 3 out of 4 violations. This is the case when using both $\allgapmax$ and $\allgapmaxbinarydiff$ as the finetuning objective. We also include training and validation curves in Figure~\ref{fig:MNIST_fntn_binarydiff_curves}.

\paragraph{COMPAS.} We provide here the training and validation curves for the \texttt{COMPAS} experiments elaborated on in the main text. See Figure~\ref{fig:COMPAS_fntn_binarydiff_curves}.

\paragraph{NELS.} The results obtained by the $\allgapmax$ and $\allgapmaxbinarydiff$ are even closer for the \texttt{NELS} dataset. With either approach, we find a predictor that is equivalent in terms of accuracy, but offers much smaller PPV violation. We do so by trading off a relatively smaller increase if FPR, and where TPR and NPV are equivalent, as can be seen in Figure~\ref{fig:fntn_binarydiff_full} in the main text. A similar dynamic can be seen in Figure~\ref{fig:fntn_maxDelta_full}. We also include training and validation curves in Figure~\ref{fig:NELS_fntn_binarydiff_curves}.